\documentclass[preprint,12pt]{elsarticle}
\usepackage[utf8]{inputenc}
\usepackage{amsmath,amssymb,amsthm}
\usepackage{geometry}
\usepackage{physics}
\usepackage{booktabs}
\usepackage{stmaryrd}
\newtheorem{definition}{Definition}[section]
\newtheorem{theorem}{Theorem}[section]
\newtheorem{proposition}{Proposition}[section]

\usepackage{enumitem}
\usepackage{url}

\usepackage{tikz}
\usetikzlibrary{positioning,arrows.meta,fit,calc}

\DeclareMathOperator{\Med}{Med}

\journal{Information Sciences}

\begin{document}

\begin{frontmatter}



\title{Mediative Fuzzy Logic: From Type-1 Foundations to Type-2, Type-3 and Quantum Extensions}


\author{Oscar Montiel Ross} 
\ead{oross@ipn.mx}
\affiliation{organization={Instituto Politécnico Nacional - CITEDI},
            addressline={AV. Instituto Politécnico Nacional 1310}, 
            city={Tijuana},
            postcode={22435}, 
            state={Baja California},
            country={México}}

\begin{abstract}
Mediative Fuzzy Logic was conceived as a practical scheme for reconciling hesitant or conflicting assessments in fuzzy control and decision-making. However, its logical and semantic foundations remain underdeveloped, especially beyond operational type-1 settings. This article develops a unified account of the type-1 core together with interval type-2, granular type-3, and quantum extensions. We characterize the mediative operator as a convex aggregation controlled by hesitation and contradiction, model mediative truth values as independent truth-falsity pairs in a continuous bilattice-like structure, and introduce a propositional system extending a standard t-norm-based fuzzy logic with a mediative connective. We establish soundness, paraconsistency, and conservativity over the underlying fuzzy base for formulas without mediation, and formulate coherent semantic extensions to interval type-2 truth values, granule-indexed local evaluations, and effects and density operators on Hilbert spaces. An autonomous-braking sensor-fusion example illustrates how the framework supports transparent, conservative, and safety-first decisions under incomplete, heterogeneous, and mildly contradictory evidence. Under suitable assumptions, the higher-level formulations reduce to the type-1 case, clarifying coherence across levels and reliably supporting future work in intelligent decision systems.
\end{abstract}




\end{frontmatter}



\section{Introduction} \label{sec:introduction}
Real-world decision problems are typically based on information that is not only imprecise but also incomplete and, in many cases, genuinely contradictory. Classical fuzzy logic is well adapted to graded truth~\cite{Zadeh1965,Zimmermann2011,KlirYuan1995}, and its logical foundations are now well established~\cite{Hajek1998}. However, in the standard setting, falsity is not an independent degree: it is obtained by applying a fixed fuzzy negation $N$ to the truth degree, typically the standard negation $N(\mu)=1-\mu$. This assumption becomes too restrictive when different experts, sensors, or models provide conflicting assessments of the same proposition.

To relax this constraint, intuitionistic fuzzy sets and related frameworks separate truth and falsity degrees and introduce an additional component of hesitation~\cite{Atanassov1986,Atanassov1999,Atanassov2020}. Such approaches naturally represent incomplete information; yet, they are not primarily tailored to explicitly capture and reason with persistent contradictions across sources. Persistent inconsistency is more naturally addressed within paraconsistent and bilattice-based logics~\cite{Belnap1977,Ginsberg1988,ArieliAvron1996}, which support controlled reasoning in the presence of conflicting evidence. In parallel, intuitionistic fuzzy sets and their aggregation operators are widely used in information fusion as a principled framework for combining uncertain and potentially conflicting evidence from multiple sources~\cite{Atanassov2014InformationFusion,Xu2010InformationFusion}.

Mediative Fuzzy Logic (MFL) was first introduced in an operational form as a scheme that combines an agreement channel and a non-agreement channel into a single mediative output~\cite{Montiel2008MFL,Montiel2009MFLPopSize}. In that formulation, each proposition is described by a membership function (agreement), a non-membership function (non-agreement), and the associated hesitation and contradiction fuzzy sets. The mediative operator then aggregates the outputs of the two subsystems so that the resulting value jointly reflects an intuitionistic-style hesitation margin and a contradiction index induced by the contradiction fuzzy set~$C$. In applications to control and diagnosis, MFL has been shown to handle both hesitation and contradiction in a smooth and interpretable manner, with successful implementations in medical diagnosis and pandemic modeling~\cite{Iancu2018HeartMFL,SharmaEtAl2021CovidModelMFL,SharmaEtAl2022CovidProjectionMFL}. Nevertheless, the original formulation was not cast as a fully axiomatized logic with explicit algebraic semantics, and most existing works treat MFL primarily as a fuzzy inference scheme rather than as a proof-theoretic logic~\cite{CastilloMelin2023MFLControl,MelinCastillo2025T3MFL}.

At a foundational level, our aim is to provide algebraic and logical foundations for Mediative Fuzzy Logic beyond the original operational control setting. Starting from a type-1 mediative operator with clear algebraic semantics, we define a propositional logic MFL-T1 and develop its basic metatheory, taking as fuzzy base a standard t-norm-based fuzzy logic, such as H{\'a}jek's Basic Logic (BL)~\cite{Hajek1998BFL} or {\L}ukasiewicz logic~\cite{Hajek1998}. Our guiding question is whether mediative fuzzy semantics can reconcile paraconsistency and safety-first aggregation while remaining compatible with standard fuzzy infrastructures.

Building on this type-1 core, we ask how mediative reasoning behaves at higher types. Type-2 and type-3 Mediative Fuzzy Logic (MFL-T2 and MFL-T3) were introduced by Castillo and Melin in their proposals for mediative fuzzy control and type-3 mediative systems~\cite{CastilloMelin2023MFLControl,MelinCastillo2025T3MFL,MelinCastillo2024INFUS}. These contributions outline mediative architectures from type-1 to type-3 and report several control applications; however, they are formulated primarily at the level of system design and do not provide a granular semantics or a logical calculus in which higher-type mediative truth values are treated as structured objects.

In this paper, we develop a systematic mediative semantics for interval type-2 truth values $(\tilde{\mu}_p,\tilde{\nu}_p)$ within an MFL-T2 framework. In this setting, type-2 hesitation and contradiction are made explicit, and a mediative evaluation $\tilde{M}_p$ is obtained via standard type-reduction mechanisms used in type-2 fuzzy decision models~\cite{KM2001,BaskovNoghin2022General,BaskovNoghin2022Implementations}. In contrast to previous combinations of mediative fuzzy logic with type-2 fuzzy controllers and higher-type mediative systems~\cite{CastilloMelin2023MFLControl,MelinCastillo2024INFUS,MelinCastillo2025T3MFL}, or with general type-2 fuzzy logic systems and their decision models~\cite{KM1999,Mendel2001}, our construction treats $(\tilde{\mu}_p,\tilde{\nu}_p)$ as mediative truth values in their own right and links their type-2 footprints of uncertainty directly to the paraconsistent mediative operator.

This semantic viewpoint also supports a granular interpretation at type-3. For MFL-T3, we connect mediative semantics with granular computing and hierarchical reasoning~\cite{BargielaPedrycz2002,Pedrycz2013,PEDRYCZ201493} and introduce explicit aggregation operators over families of local mediative evaluations indexed by granules (e.g., experts, sensors, or time slices). Accordingly, MFL-T2 captures second-order uncertainty, whereas MFL-T3 accounts for hierarchical evidence in multi-source and multi-level settings, and both remain conceptually coherent with the type-1 core. We complement these higher-type semantics with an explicit logical treatment that keeps the type-1 core intact while clarifying how type-2 uncertainty and type-3 heterogeneity propagate through connectives.

Beyond higher-type fuzzy sets, we also introduce Quantum Mediative Fuzzy Logic (QMFL), whose semantics is given in terms of quantum effects and density operators. Our construction relates mediative truth degrees to effect algebras in the sense of Foulis and Bennett~\cite{FoulisBennett1994} and to fuzzy approaches to quantum logics~\cite{Pykacz1992,Pykacz2015}. The logical perspective is informed by many-valued and modal quantum logics~\cite{DallaChiaraGiuntini2002,AldanaLledo2023} and by algebraic and lattice-theoretic studies of effect algebras and $L$-valued quantum structures~\cite{NavaraPtak1998,NavaraPtak1999,ChajdaLaenger2020Residuation}. We also build on recent work on axiomatic treatments of effect-algebra-based logics~\cite{WangWuYang2019,ChajdaHalasLaenger2020}.

The contributions of this paper can be summarized as follows.
\begin{itemize}
  \item We axiomatize the mediative operator as a convex aggregation controlled by hesitation and contradiction parameters, and we establish basic properties such as boundedness between its two input degrees (the outputs of the agreement and non-agreement channels) and reductions to type-1 and intuitionistic-fuzzy combinations.

  \item We introduce mediative truth values as pairs $(\mu,\nu)\in[0,1]^2$ of truth and falsity degrees, from which hesitation $\pi$ and contradiction $\zeta$ are derived, and we endow this space with suitable conjunction, disjunction, and negation operations, together with truth and information orders, so that $[0,1]^2$ carries a continuous bilattice-like structure.

  \item We define a propositional logic MFL-T1 with standard fuzzy connectives plus a mediative connective, and we provide a Hilbert-style axiom system that extends a chosen fuzzy base logic, such as H\'ajek's Basic Logic (BL) or {\L}ukasiewicz logic.

  \item We prove that MFL-T1 is sound for its mediative semantics, paraconsistent in a precise sense, and a conservative extension of the underlying fuzzy logic on formulas without the mediative operator.

  \item We extend the mediative semantics to type-2 fuzzy sets (MFL-T2), modeling second-order uncertainty about truth, falsity, hesitation, and contradiction, and we discuss how standard type-reduction methods can be used to interpret mediative degrees under second-order uncertainty.

  \item We propose a type-3 granular extension (MFL-T3) that organizes mediative truth into multi-level granular structures indexed by arbitrary granules (e.g., experts, sensors, or time slices) within a unified framework.

  \item We introduce Quantum Mediative Fuzzy Logic (QMFL), whose semantics is formulated in terms of quantum effects and states on Hilbert spaces, and we illustrate a simple embedding of type-1 mediative degrees into QMFL.

  \item We illustrate the safety-first interpretation of mediative truth degrees through a detailed case study of sensor fusion for obstacle detection and braking decisions in autonomous driving.

  \item We clarify how standard type-1 fuzzy and intuitionistic fuzzy semantics arise as special cases of MFL-T1 by imposing suitable constraints on mediative truth values and parameters, under which the mediative operator reduces to the usual $t$-norm-based combinations.

  \item We establish reduction results that connect the different levels of the framework: granular MFL-T3 reduces to MFL-T2 or MFL-T1 when granules become homogeneous, and QMFL reduces to the classical mediative semantics when the relevant effects commute and the quantum states are diagonal in a common basis, thereby ensuring multi-level coherence.
\end{itemize}

Existing works on Mediative Fuzzy Logic~\cite{CastilloMelin2023MFLControl} and its type-3 extension~\cite{MelinCastillo2025T3MFL} do not develop, to the best of our knowledge, the kind of granular semantics used here, where mediative truth values are treated as families of local evaluations indexed by arbitrary granules (e.g., experts, sensors, or time slices) and combined through explicit aggregation operators. Accordingly, the definitions in Section~6 provide a first granular formalization of MFL-T3 at the semantic level. In addition, the quantum extension proposed here makes the connection with effect-algebraic semantics explicit, aligning mediative reasoning with effect-based viewpoints in quantum granular computing~\cite{ross2025foundationsquantumgranularcomputing}. To our knowledge, there is also no prior work that treats interval type-2 pairs as mediative truth values in their own right.

The paper is organized as follows. Section~\ref{sec:background} reviews the intuitionistic-fuzzy and paraconsistent background. Section~\ref{sec:type1-operator} introduces the type-1 mediative operator and its algebraic semantics. Section~\ref{sec:prop-MFL} defines the propositional logic MFL-T1 and establishes its basic metatheoretical properties. Sections~\ref{sec:MFL2}--\ref{sec:QMFL} develop MFL-T2, MFL-T3, and QMFL, respectively. Section~\ref{sec:example} presents the safety-first sensor-fusion case study, and Section~\ref{sec:conclusions} concludes. For convenience, Appendix~A collects the main notation and logical symbols used throughout the paper (Table~A.1).

\section{Background and motivation}
\label{sec:background}
In this section, we recall the intuitionistic-fuzzy and paraconsistent perspectives that motivate Mediative Fuzzy Logic. The goal is not to provide an exhaustive survey, but to highlight aspects directly relevant to the mediative semantics developed in the subsequent sections.

In classical fuzzy logic, each proposition $p$ is assigned a truth degree $\mu_p\in[0,1]$, and the degree of falsity is determined by a fixed negation, typically $1-\mu_p$~\cite{Zimmermann2011}. In this setting, uncertainty is collapsed into a single dimension: once $\mu_p$ is fixed, there is no independent degree of falsity or inconsistency. In contrast, intuitionistic fuzzy sets associate each element or proposition $p$ with a pair $(\mu_p,\nu_p)$ of truth and falsity degrees, subject to the constraint $\mu_p+\nu_p\le 1$~\cite{Atanassov1986,Atanassov1999,Atanassov2020}. The remaining amount $\pi_p = 1-\mu_p-\nu_p$ is interpreted as hesitation (or indeterminacy) and explicitly captures incomplete information.

Paraconsistent logics, on the other hand, admit the possibility that both $p$ and $\neg p$ may hold to significant degrees without collapsing into triviality~\cite{Ripley2015Paraconsistent}. Bilattice-based logics, such as the Dunn--Belnap logic of first-degree entailment, provide a well-known example in which each proposition is evaluated by two coordinates capturing truth and knowledge (or information)~\cite{Belnap1977,Ginsberg1988,ArieliAvron1996}. These frameworks support controlled reasoning in the presence of contradictions, but they are not usually formulated in terms of fuzzy membership and non-membership functions.

Mediative Fuzzy Logic (MFL) was originally proposed in an operational form to combine an agreement channel and a non-agreement channel into a single mediative output~\cite{Montiel2008MFL}. In that formulation, each proposition or element is described by an agreement membership function and a non-agreement membership (or non-membership) function, and the mediative operator aggregates the corresponding fuzzy-system outputs. In applications to control and diagnosis, MFL has been shown to handle both hesitation and contradiction in a smooth and interpretable manner, with successful implementations in medical diagnosis and pandemic modeling~\cite{Iancu2018HeartMFL,SharmaEtAl2021CovidModelMFL,SharmaEtAl2022CovidProjectionMFL}. Nevertheless, the original formulation was not cast as a fully axiomatized logic with clear algebraic semantics. Most existing works treat MFL as a fuzzy inference scheme rather than as a logic in the proof-theoretic sense~\cite{CastilloMelin2023MFLControl,MelinCastillo2025T3MFL}. In particular, when only hesitation is present, MFL reduces to an intuitionistic-style combination, whereas in the presence of contradiction, the mediative output systematically takes an intermediate value between the agreement and non-agreement channels.

The following sections formalize this informal picture into a precise algebraic and logical framework and extend mediative reasoning to type-2, type-3, and quantum settings.

\section{Type-1 Mediative Fuzzy Logic: operator and algebraic semantics}
\label{sec:type1-operator}

We study the mediative operator at an abstract level, independently of any particular logical language. Our starting point is a pair of numerical outputs $a,b\in[0,1]$, typically generated by an agreement channel and a non-agreement channel, together with two parameters $\pi,\zeta\in[0,1]$ representing hesitation and contradiction, respectively.

\begin{definition}[Mediative operator]
\label{def:mediative-operator}
For $a,b\in[0,1]$ and parameters $\pi,\zeta\in[0,1]$, define
\[
  \mathcal{M}(a,b;\pi,\zeta)
  := \Bigl(1-\pi-\frac{\zeta}{2}\Bigr)a + \Bigl(\pi+\frac{\zeta}{2}\Bigr)b.
\]
\end{definition}

We impose the following axioms on $\mathcal{M}$, which constrain the admissible parameters $(\pi,\zeta)$:
\begin{itemize}
  \item[M1] \textbf{(weight normalization).} Let $w_1 = 1-\pi-\zeta/2$ and $w_2 = \pi+\zeta/2$.
  Then $0 \le w_1,w_2 \le 1$ and $w_1+w_2=1$.

  \item[M2] \textbf{(intuitionistic reduction).} If $\zeta = 0$, then
  \begin{equation}
    \label{eq:Med-intuitionistic}
    \mathcal{M}(a,b;\pi,0) = (1-\pi)\,a + \pi\,b,
  \end{equation}
  so that the mediative operator reduces to a convex combination controlled only by the hesitation parameter $\pi$.

  \item[M3] \textbf{(type-1 reduction).} If $\pi = 0$ and $\zeta = 0$, then
  \begin{equation}
    \label{eq:Med-type1}
    \mathcal{M}(a,b;0,0) = a,
  \end{equation}
  that is, the mediative operator reduces to the agreement channel's output.
\end{itemize}

The next results formalize basic properties of $\mathcal{M}$.

\begin{theorem}[Convexity and boundedness]
\label{thm:convexity}
Let $w_1 = 1-\pi-\zeta/2$ and $w_2 = \pi+\zeta/2$, and suppose that $w_1,w_2 \ge 0$ and $w_1+w_2=1$ (axiom M1). Then, for all $a,b \in [0,1]$, we have
\begin{equation}
  \label{eq:Med-boundedness}
  \min(a,b) \;\le\; \mathcal{M}(a,b;\pi,\zeta) \;\le\; \max(a,b).
\end{equation}
\end{theorem}

\begin{proof}[Proof sketch]
Under Axiom~(M1), we have $w_1,w_2 \ge 0$ and $w_1+w_2=1$, hence
\[
  \mathcal{M}(a,b;\pi,\zeta) = w_1 a + w_2 b
\]
is a convex combination of $a$ and $b$. Any convex combination of two real numbers lies between their minimum and maximum, which yields~\eqref{eq:Med-boundedness}.
\end{proof}

By Axiom~(M1), the admissible parameter pairs $(\pi,\zeta)$ are precisely those for which the induced weights satisfy $w_1,w_2\in[0,1]$ and $w_1+w_2=1$ (equivalently, $\pi \ge 0$, $\zeta \ge 0$, and $\pi+\zeta/2 \le 1$). Hence, the mediative operator is a well-defined convex aggregation of the two channels $a$ and $b$ and cannot extrapolate beyond the interval spanned by its inputs. Moreover, in the induced type-1 semantics we have $\pi(\mu,\nu)=\max\{0,1-\mu-\nu\}$ and $\zeta(\mu,\nu)=\max\{0,\mu+\nu-1\}$ for $(\mu,\nu)\in[0,1]^2$, so $\pi,\zeta\in[0,1]$ and, in particular, they are nonnegative degrees.

\begin{theorem}[Reductions]
\label{thm:reductions}
For all $a,b \in [0,1]$, the mediative operator $\mathcal{M}$ satisfies:
\begin{enumerate}
  \item If $\zeta = 0$, then~\eqref{eq:Med-intuitionistic} holds. In other words, the mediative operator reduces to the intuitionistic-style convex combination controlled solely by the hesitation degree $\pi$.

  \item If $\pi = \zeta = 0$, then~\eqref{eq:Med-type1} holds, recovering the underlying type-1 fuzzy evaluation (the agreement-channel output).
\end{enumerate}
\end{theorem}

\begin{proof}
Both items follow directly by substituting the corresponding parameter values into the definition of $\mathcal{M}$.
\end{proof}

\begin{theorem}[Effect of contradiction]
\label{thm:contradiction-effect}
Let $a,b \in [0,1]$ and assume $\pi = 0$ and $0 < \zeta < 1$. Then
\[
  \mathcal{M}(a,b;0,\zeta) = \Bigl(1-\frac{\zeta}{2}\Bigr)a + \Bigl(\frac{\zeta}{2}\Bigr)b.
\]

In particular, if $a > b$, then
\[
  a > \mathcal{M}(a,b;0,\zeta) > b.
\]
By symmetry, if $a < b$, then
\[
  a < \mathcal{M}(a,b;0,\zeta) < b.
\]
\end{theorem}

\begin{proof}[Proof sketch]
For $\pi=0$, the weights become $w_1 = 1-\zeta/2$ and $w_2 = \zeta/2$, both strictly between $0$ and $1$ when $0 < \zeta < 1$. Hence
\[
  \mathcal{M}(a,b;0,\zeta) = w_1 a + w_2 b
\]
is a strict convex combination of $a$ and $b$. If $a > b$, any strict convex combination $w_1 a + w_2 b$ with $w_1,w_2 \in (0,1)$ lies strictly between $a$ and $b$, yielding $a > \mathcal{M}(a,b;0,\zeta) > b$. The case $a < b$ is analogous.
\end{proof}

As ensured by Axiom~(M1), $\mathcal{M}(a,b;\pi,\zeta)$ is a convex aggregation and cannot extrapolate beyond the interval spanned by $a$ and $b$. In the instantiated semantics where $\pi=\pi(\mu,\nu)$ and $\zeta=\zeta(\mu,\nu)$ are obtained by a positive-part construction, these quantities are nonnegative by definition.

\subsection{Mediative truth values and algebraic semantics}
\label{sec:med-truth-values}

We now introduce an algebraic semantics for type-1 MFL in terms of pairs of truth and falsity degrees. These pairs constitute mediative truth values and provide the basis for defining hesitation and contradiction in a uniform way, consistent with the agreement/non-agreement interpretation of MFL.

Let $V = [0,1]\times[0,1]$. A mediative truth value is any pair $(\mu,\nu)\in V$, where $\mu$ is interpreted as a degree of truth (or agreement) and $\nu$ as a degree of falsity (or non-agreement). From each pair, we derive two secondary quantities, hesitation and contradiction.

\begin{definition}[Hesitation and contradiction]
\label{def:hes-contr}
For $(\mu,\nu) \in V$ we define
\[
  \pi(\mu,\nu) = \max(0, 1 - \mu - \nu), \qquad
  \zeta(\mu,\nu) = \max(0, \mu + \nu - 1).
\]
Thus $\pi(\mu,\nu) > 0$ precisely when $\mu + \nu < 1$ (incomplete information), whereas $\zeta(\mu,\nu) > 0$ precisely when $\mu + \nu > 1$ (overdetermined, possibly contradictory information).
\end{definition}

Note that $\pi(\mu,\nu)$ and $\zeta(\mu,\nu)$ cannot be simultaneously positive. This cleanly separates hesitation from contradiction and will be useful when interpreting the mediative operator in terms of agreement and non-agreement channels.

To endow $V$ with logical operations, we fix a left-continuous $t$-norm $T$ and its dual $t$-conorm $S$, and define conjunction, disjunction, and negation coordinatewise on mediative truth values.

\begin{definition}[Conjunction, disjunction and negation on $V$]
\label{def:V-ops}
Fix a left-continuous $t$-norm $T$ and its dual $t$-conorm $S$. For
$(\mu_1,\nu_1),(\mu_2,\nu_2)\in V=[0,1]^2$, define
\[
(\mu_1,\nu_1)\wedge(\mu_2,\nu_2):=\bigl(T(\mu_1,\mu_2),\,S(\nu_1,\nu_2)\bigr),
\]
\[
(\mu_1,\nu_1)\vee(\mu_2,\nu_2):=\bigl(S(\mu_1,\mu_2),\,T(\nu_1,\nu_2)\bigr),
\]
and
\[
\neg(\mu,\nu):=(\nu,\mu).
\]
\end{definition}

Conjunction decreases the truth coordinate and increases the falsity coordinate, whereas disjunction increases truth and decreases falsity. Negation swaps the truth and falsity coordinates, in accordance with the intuition that the falsity of $p$ corresponds to the truth of $\neg p$. In this way, $(V,\wedge,\vee,\neg)$ forms a bilattice-like structure over pairs of truth and falsity degrees, on top of which the mediative evaluation is defined.

\begin{definition}[Mediative evaluation]
\label{def:mediative-evaluation}
Let $(\mu,\nu)\in V$. The mediative evaluation of $(\mu,\nu)$ is the scalar
\[
  M(\mu,\nu)
  := \mathcal{M}\!\bigl(a(\mu,\nu),\, b(\mu,\nu);\, \pi(\mu,\nu),\, \zeta(\mu,\nu)\bigr),
\]
where, in the basic case, we set $a(\mu,\nu)=\mu$ and $b(\mu,\nu)=1-\nu$.
\end{definition}

Thus, $M(\mu,\nu)$ is the mediative score associated with a proposition whose degrees of agreement and non-agreement are $\mu$ and $\nu$, respectively. It combines the agreement channel $a(\mu,\nu)=\mu$ and the lack-of-disagreement channel $b(\mu,\nu)=1-\nu$, with mixing weights given by the hesitation $\pi(\mu,\nu)$ and the contradiction $\zeta(\mu,\nu)$ extracted from $(\mu,\nu)$. The choices $a(\mu,\nu)=\mu$ and $b(\mu,\nu)=1-\nu$ reflect the idea that the mediative score should increase with agreement and and decrease with disagreement about $p$.

\section{Propositional Mediative Fuzzy Logic (MFL-T1)}
\label{sec:prop-MFL}
We define a propositional logic, MFL-T1, whose semantics is based on mediative truth values.

As usual, we use the definable constants $\top := (p \to p)$ and $\bot := \neg \top$, for an arbitrary fixed atomic proposition $p$.

\begin{definition}[Language]
\label{def:language}
The language of MFL-T1 extends that of the chosen fuzzy base logic (BL or {\L}ukasiewicz logic): it includes the connectives $\wedge,\vee,\neg,\to$, together with a unary mediative connective $\Med$. We write $\Med\,\varphi$ for its application to a formula $\varphi$.
\end{definition}

\begin{definition}[Implication on $V$]
\label{def:implication-on-V}
Fix a left-continuous $t$-norm $T$ on $[0,1]$ and let $\Rightarrow_T$ denote its residuum.
For $(\mu_1,\nu_1),(\mu_2,\nu_2)\in V=[0,1]^2$, define
\[
  (\mu_1,\nu_1)\to(\mu_2,\nu_2)
  :=
  \bigl(\mu_1 \Rightarrow_T \mu_2,\; \nu_2 \Rightarrow_T \nu_1\bigr).
\]
\end{definition}

\begin{definition}[Valuations]
\label{def:valuations}
A mediative valuation on the language of MFL-T1 is a function $v$ that assigns to each formula $\varphi$ a mediative truth value $v(\varphi)\in V$ and satisfies, for all formulas $\varphi,\psi$:
\begin{itemize}
  \item[(V1)] For each atomic proposition $p$, the value $v(p)$ is arbitrary in $V$.
  \item[(V2)] $v(\varphi\wedge\psi)=v(\varphi)\wedge v(\psi)$.
  \item[(V3)] $v(\varphi\vee\psi)=v(\varphi)\vee v(\psi)$.
  \item[(V4)] $v(\neg\varphi)=\neg v(\varphi)$.
  \item[(V5)] $v(\varphi\to\psi)=v(\varphi)\to v(\psi)$, where $\to$ on $V$ is as in Definition~\ref{def:implication-on-V}.
  \item[(V6)] $v(\Med(\varphi))=\bigl(M(v(\varphi)),\,1-M(v(\varphi))\bigr)$.
\end{itemize}
\end{definition}

\subsection{Axiomatic system}
\label{sec:axiomatic-system}
We build MFL-T1 on top of a fuzzy base logic such as H\'ajek's Basic Logic (BL) or {\L}ukasiewicz logic~\cite{Hajek1998}, which provides the axioms and rules for the connectives $\wedge,\vee,\to,\neg$. We then add axiom schemata governing the mediative connective.

\subsubsection{Axiom schemata for $\Med$}
We use the abbreviation $\varphi \leftrightarrow \psi := (\varphi \to \psi)\wedge(\psi \to \varphi)$.

\begin{enumerate}[label=\textup{(Med\arabic*)}, leftmargin=*, align=left, labelsep=.6em, itemsep=.4ex, topsep=.6ex]
  \item $(\varphi \to \psi) \to (\Med(\varphi) \to \Med(\psi)).$
  \item $\Med(\top) \leftrightarrow \top \ \ \text{and}\ \ \Med(\bot) \leftrightarrow \bot.$
  \item $(\varphi \leftrightarrow \psi) \to (\Med(\varphi) \leftrightarrow \Med(\psi)).$
\end{enumerate}

\subsection{Metatheoretical properties of MFL-T1}
\label{sec:metatheory-MFL1}
We record several basic metatheoretical properties of MFL-T1. The proofs rely on standard techniques from fuzzy logic and algebraic logic and are omitted. For the notation used below, see Appendix~A and Table~\ref{tab:mfl-symbols-extended}.

\paragraph{Soundness}
MFL-T1 is sound with respect to the mediative semantics: for every set $\Gamma$ of formulas and every formula $\varphi$, if $\Gamma \vdash_m \varphi$, then $\Gamma \models_m \varphi$. Equivalently, whenever $\Gamma \vdash_m \varphi$ and $v$ is a mediative valuation such that $M(v(\psi)) = 1$ for all $\psi \in \Gamma$, we also have $M(v(\varphi)) = 1$. In particular, if $\vdash_m \varphi$, then $\models_m \varphi$, i.e., $M(v(\varphi)) = 1$ for every mediative valuation $v$.

\paragraph{Paraconsistency}
MFL-T1 is paraconsistent: there exist formulas $\varphi$ and mediative valuations $v$ such that both $M(v(\varphi))$ and $M(v(\neg \varphi))$ can take high values simultaneously, while the explosion principle
\[
  (\varphi \wedge \neg \varphi) \to \psi
\]
is not derivable in MFL-T1 for arbitrary $\psi$.

\paragraph{Reduction to the base fuzzy logic}
Let $\varphi$ be a formula that does not contain the mediative connective $\Med$. Then $\varphi$ is derivable in MFL-T1 if and only if it is derivable in the underlying fuzzy base logic (BL or {\L}ukasiewicz logic). Thus, MFL-T1 is a conservative extension of the chosen base logic.

\paragraph{Reduction to intuitionistic fuzzy and type-1 cases}
If mediative valuations are restricted so that $\mu + \nu \le 1$ for all atomic propositions, the mediative evaluation reduces to an intuitionistic-style combination controlled only by hesitation. In the further special case where $\nu = 1 - \mu$ for all atoms, the mediative evaluation coincides with the underlying type-1 fuzzy evaluation, and MFL-T1 reduces to the base fuzzy logic at the semantic level.

\section{Type-2 Mediative Fuzzy Logic (MFL-T2)}
\label{sec:MFL2}

This section develops a semantic treatment of Type-2 Mediative Fuzzy Logic (MFL-T2) along the lines proposed in~\cite{CastilloMelin2023MFLControl}. The key point is that the primary degrees of truth and falsity are themselves uncertain; we model this second-order variability---due to noise, calibration drift, or changing environmental conditions---by interval type-2 fuzzy sets. For simplicity, we restrict attention to the interval case, which already captures the main mechanisms we need.

\begin{definition}[Footprint of uncertainty (FOU)~\cite{Mendel2001}]
\label{def:FOU}
Let $\tilde{A}$ be an interval type-2 fuzzy set on $[0,1]$ with lower and upper membership
functions $A^L, A^U : [0,1] \to [0,1]$. The footprint of uncertainty (FOU) of $\tilde{A}$ is
the set 
\[
  \mathrm{FOU}(\tilde{A})
    = \bigl\{\, (x,u) \in [0,1] \times [0,1] : A^L(x) \le u \le A^U(x) \,\bigr\}.
\]
In the present setting, $\mathrm{FOU}(\tilde{A})$ collects all admissible pairs $(x,u)$
where $x\in[0,1]$ is a candidate primary degree and $u\in[0,1]$ is an admissible membership
level consistent with the lower and upper membership functions.
\end{definition}

At the type-1 level, each proposition $p$ is assigned a mediative truth value $(\mu_p,\nu_p)$.
At the type-2 level, these scalar degrees are replaced by interval type-2 fuzzy sets, so that
the truth value of $p$ becomes a pair $(\tilde{\mu}_p,\tilde{\nu}_p)$, equivalently the pair of
footprints of uncertainty $\mathrm{FOU}(\tilde{\mu}_p)$ and $\mathrm{FOU}(\tilde{\nu}_p)$.

\begin{definition}[Type-2 mediative truth values]
\label{def:type2-mediative-values}
A type-2 mediative truth value for an atomic proposition $p$ is a pair
\[
  v^{(2)}(p) = (\tilde{\mu}_p,\tilde{\nu}_p),
\]
where $\tilde{\mu}_p$ and $\tilde{\nu}_p$ are interval type-2 fuzzy sets on $[0,1]$ with lower and upper membership functions $\mu_p^L,\mu_p^U$ and $\nu_p^L,\nu_p^U$, respectively. Here, the superscripts $L,U$ refer to lower/upper membership functions, whereas the
underline/overline notation $\underline{\mu}_p,\overline{\mu}_p$ and $\underline{\nu}_p,\overline{\nu}_p$
is reserved for the projected scalar bounds induced by these memberships (Definition~\ref{def:proj-bounds}).

Equivalently, $v^{(2)}(p)$ can be represented by the pair of footprints of uncertainty
$\mathrm{FOU}(\tilde{\mu}_p)$ and $\mathrm{FOU}(\tilde{\nu}_p)$.
By ``atomic proposition'' we mean a propositional variable (a formula with no connectives);
this valuation is extended inductively to compound formulas by the semantic clauses for the
connectives.
\end{definition}

\begin{definition}[Projected interval bounds]
\label{def:proj-bounds}
Let $\tilde{A}$ be an interval type-2 fuzzy set on $[0,1]$ with lower and upper membership
functions $A^L$ and $A^U$. Assume $A^U\not\equiv 0$. We define the \emph{outer} (conservative)
projection of $\tilde{A}$ by projecting the support of $A^U$:
\[
  \mathrm{Proj}^{U}(\tilde{A})=[\underline{a}^{U},\overline{a}^{U}],\qquad
  \underline{a}^{U}=\inf\{x\in[0,1]: A^U(x)>0\},\quad
  \overline{a}^{U}=\sup\{x\in[0,1]: A^U(x)>0\}.
\]
If $A^L\not\equiv 0$, we also define the \emph{inner} (guaranteed) projection by projecting the
support of $A^L$:
\[
  \mathrm{Proj}^{L}(\tilde{A})=[\underline{a}^{L},\overline{a}^{L}],\qquad
  \underline{a}^{L}=\inf\{x\in[0,1]: A^L(x)>0\},\quad
  \overline{a}^{L}=\sup\{x\in[0,1]: A^L(x)>0\}.
\]
More generally, one may replace these support-based projections by $\alpha$-cut projections,
using $\{x: A^U(x)\ge \alpha\}$ and $\{x: A^L(x)\ge \alpha\}$ for a prescribed $\alpha\in(0,1]$.
Once an interval has been selected (outer, inner, or $\alpha$-cut based), subsequent steps treat
it as the admissible range of the primary degree and propagate it through the same interval
computations used later (for $\pi$, $\zeta$, and the mediative evaluation).
\end{definition}

\begin{definition}[Interval bounds for type-2 connectives]
\label{def:type2-connectives}
Let $\varphi,\psi$ be formulas, and assume each is associated with projected scalar bounds
\[
  \mu_\varphi \in [\underline{\mu}_\varphi,\overline{\mu}_\varphi],\quad
  \nu_\varphi \in [\underline{\nu}_\varphi,\overline{\nu}_\varphi],
  \qquad
  \mu_\psi \in [\underline{\mu}_\psi,\overline{\mu}_\psi],\quad
  \nu_\psi \in [\underline{\nu}_\psi,\overline{\nu}_\psi],
\]
obtained, for instance, from Definition~\ref{def:proj-bounds}. Let $T$ be the $t$-norm and $S$ the
associated $t$-conorm used in Definition~\ref{def:V-ops}. Let $\Rightarrow_T$ denote the residuum of $T$
as in Definition~\ref{def:implication-on-V}.

Since $T$ and $S$ are nondecreasing in each argument, conservative bounds for $\wedge$ and $\vee$ are obtained by endpoint evaluation. For $\neg$, we use the bilattice negation (swap of truth and falsity bounds). For $\to$, recall that the residuum $\Rightarrow_T$ is antitone in its first argument and monotone in its second; conservative bounds again follow from endpoint evaluation using the worst-case combinations of premise and conclusion bounds.

\begin{equation}
\label{eq:type2-bounds-and-or-not}
\begin{aligned}
\underline{\mu}_{\varphi\wedge\psi} &= T(\underline{\mu}_\varphi,\underline{\mu}_\psi), &
\overline{\mu}_{\varphi\wedge\psi} &= T(\overline{\mu}_\varphi,\overline{\mu}_\psi),\\
\underline{\nu}_{\varphi\wedge\psi} &= S(\underline{\nu}_\varphi,\underline{\nu}_\psi), &
\overline{\nu}_{\varphi\wedge\psi} &= S(\overline{\nu}_\varphi,\overline{\nu}_\psi),\\[3pt]
\underline{\mu}_{\varphi\vee\psi} &= S(\underline{\mu}_\varphi,\underline{\mu}_\psi), &
\overline{\mu}_{\varphi\vee\psi} &= S(\overline{\mu}_\varphi,\overline{\mu}_\psi),\\
\underline{\nu}_{\varphi\vee\psi} &= T(\underline{\nu}_\varphi,\underline{\nu}_\psi), &
\overline{\nu}_{\varphi\vee\psi} &= T(\overline{\nu}_\varphi,\overline{\nu}_\psi),\\[3pt]
\underline{\mu}_{\neg\varphi} &= \underline{\nu}_\varphi,\quad
\overline{\mu}_{\neg\varphi} = \overline{\nu}_\varphi, &
\underline{\nu}_{\neg\varphi} &= \underline{\mu}_\varphi,\quad
\overline{\nu}_{\neg\varphi} = \overline{\mu}_\varphi.
\end{aligned}
\end{equation}

For implication, we use the residuum $\Rightarrow_T$ of the chosen $t$-norm $T$. Since $\Rightarrow_T$ is antitone in its antecedent and monotone in its consequent, the lower bound is attained at $(\text{first}=\overline{\phantom{\mu}},\,\text{second}=\underline{\phantom{\mu}})$ and the upper bound at $(\text{first}=\underline{\phantom{\mu}},\,\text{second}=\overline{\phantom{\mu}})$:
\begin{equation}
\label{eq:type2-imp-bounds}
\begin{aligned}
\underline{\mu}_{\varphi\to\psi} &= \overline{\mu}_\varphi \Rightarrow_T \underline{\mu}_\psi, &
\overline{\mu}_{\varphi\to\psi} &= \underline{\mu}_\varphi \Rightarrow_T \overline{\mu}_\psi,\\
\underline{\nu}_{\varphi\to\psi} &= \overline{\nu}_\psi \Rightarrow_T \underline{\nu}_\varphi, &
\overline{\nu}_{\varphi\to\psi} &= \underline{\nu}_\psi \Rightarrow_T \overline{\nu}_\varphi.
\end{aligned}
\end{equation}

In what follows, we propagate interval bounds through compound formulas using the endpoint rules
in~\eqref{eq:type2-bounds-and-or-not} and~\eqref{eq:type2-imp-bounds}.
\end{definition}

\begin{definition}[Type-2 hesitation and contradiction (interval bounds)]
\label{def:type2-hc}
Let $\mu\in[\underline{\mu},\overline{\mu}]$ and $\nu\in[\underline{\nu},\overline{\nu}]$ be
the projected interval bounds induced by $\tilde{\mu}$ and $\tilde{\nu}$
(Definition~\ref{def:proj-bounds}). Define
\[
  H(\mu,\nu):=\max\{0,\,1-\mu-\nu\},\qquad
  C(\mu,\nu):=\max\{0,\,\mu+\nu-1\}.
\]
Since $1-\mu-\nu$ is nonincreasing in each argument and $\mu+\nu-1$ is nondecreasing in each
argument, and since $\max\{0,\cdot\}$ preserves monotonicity, conservative interval envelopes
are obtained by endpoint evaluation:
\[
  H_L=\max\{0,\,1-\overline{\mu}-\overline{\nu}\},\qquad
  H_U=\max\{0,\,1-\underline{\mu}-\underline{\nu}\},
\]
\[
  C_L=\max\{0,\,\underline{\mu}+\underline{\nu}-1\},\qquad
  C_U=\max\{0,\,\overline{\mu}+\overline{\nu}-1\}.
\]
In the type-1 semantics, $H(\mu,\nu)$ and $C(\mu,\nu)$ coincide with the hesitation $\pi$
and contradiction $\zeta$, respectively.
\end{definition}

\begin{definition}[Type-2 mediative evaluation: type-reduced and envelope modes]
\label{def:type2-med-eval}
Let $v^{(2)}(p)=(\tilde{\mu}_p,\tilde{\nu}_p)$ be an interval type-2 mediative assignment for an atomic proposition $p$. We consider two complementary ways of assigning a mediative degree.

\emph{(i) Type-reduced (crisp) mode.} First type-reduce $\tilde{\mu}_p$ and $\tilde{\nu}_p$ to a crisp pair $(\bar{\mu}_p,\bar{\nu}_p)$ (for instance, by centroid type-reduction via the Karnik--Mendel procedures), and then set
\[
  \bar{M}_p := M(\bar{\mu}_p,\bar{\nu}_p).
\]

\emph{(ii) Envelope (interval) mode.} Alternatively, extract projected bounds
$\mu_p\in[\underline{\mu}_p,\overline{\mu}_p]$ and $\nu_p\in[\underline{\nu}_p,\overline{\nu}_p]$
from the footprints of $\tilde{\mu}_p$ and $\tilde{\nu}_p$ (Definition~\ref{def:proj-bounds}), and define the envelope of admissible mediative scores by
\[
  [M_L(p), M_U(p)] := \Bigl[
\min_{\mu\in[\underline{\mu}_p,\overline{\mu}_p] \atop \nu\in[\underline{\nu}_p,\overline{\nu}_p]} M(\mu,\nu),\;
\max_{\mu\in[\underline{\mu}_p,\overline{\mu}_p] \atop \nu\in[\underline{\nu}_p,\overline{\nu}_p]} M(\mu,\nu)
\Bigr].
\]
Because $M(\mu,\nu)$ is continuous and piecewise affine (with regime boundary $\mu+\nu=1$), its extrema over a rectangle are attained on the boundary. In practice, it is enough to evaluate $M$ at the four corners of $[\underline{\mu},\overline{\mu}]\times[\underline{\nu},\overline{\nu}]$ and, whenever the rectangle intersects the line $\mu+\nu=1$, also at the intersection points with the rectangle edges.
\end{definition}

\paragraph{Extension from atoms to formulas}
In mode (i), define a type-1 valuation $\bar v(p)=(\bar\mu_p,\bar\nu_p)$ and extend it inductively
to all formulas using the same connective clauses as in MFL-T1. In mode (ii), extend the projected
bounds inductively to all formulas using Definition~\ref{def:type2-connectives}, and compute the
corresponding envelope $[M_L(\varphi),M_U(\varphi)]$ by the same optimization as above.

\begin{proposition}[Reduction to the type-1 case]
\label{prop:reduction-type1}
If for every atomic proposition $p$ the interval type-2 sets $\tilde{\mu}_p$ and $\tilde{\nu}_p$
degenerate to crisp degrees $\mu_p$ and $\nu_p$, then in both modes (i) and (ii) the induced
evaluation of formulas in MFL-T2 coincides with the type-1 semantics of MFL-T1.
\end{proposition}

\begin{proof}
By structural induction on formulas. In the base case, type-reduction returns
$(\bar\mu_p,\bar\nu_p)=(\mu_p,\nu_p)$ and the projected intervals collapse to singletons. The
inductive step follows because the connective clauses are the same as in MFL-T1 (mode (i)) or are
their monotone interval extensions (mode (ii)), which also collapse to singletons. Hence, both
modes coincide with the type-1 evaluation.
\end{proof}

\paragraph{Conservative decision rules from envelopes}
The envelope mode supports conservative decision policies: for a decision threshold $\tau$,
one may require $M_L(\varphi)\ge \tau$ to assert $\varphi$ decisively and use $M_U(\varphi)\ge \tau$
to flag $\varphi$ as potentially true and request additional evidence.

\paragraph{Axioms and proof system}
The axiomatic system introduced in Section~\ref{sec:prop-MFL} is kept unchanged. In MFL-T2, we
only enrich the semantic domain by allowing interval type-2 uncertainty in the truth and falsity
degrees assigned to atomic propositions; syntactic derivability is defined exactly as in 
type-1 logic.

\section{Type-3 Granular Mediative Fuzzy Logic (MFL-T3)}
\label{sec:MFL3}

Type-3 Granular Mediative Fuzzy Logic (MFL-T3), originally proposed in~\cite{CastilloMelin2023MFLControl}, extends mediative reasoning to higher-order uncertainty by combining it with type-3 fuzzy systems. The goal is to represent knowledge from multiple experts and to capture variability arising from heterogeneous evidence streams (e.g., distributed sensing). Here we provide a granular, multi-level formalization in which mediative truth is organized into families indexed by \emph{granules} (e.g., information sources, evidence modalities, time slices, sensor channels, or expert groups)~\cite{PEDRYCZ201493,QIN2023101833}.

In MFL-T3, a truth value is not treated as a single mediative pair (type-1) nor as a single footprint of uncertainty (type-2), but as a structured family of local evaluations indexed by granules. This representation captures heterogeneity across sources and contexts and supports principled aggregation of potentially conflicting evidence within the mediative semantics.

Let $G$ be a finite nonempty set of granules (e.g., expert, sensor, and time triples). For each granule $g \in G$ we consider a \emph{local} mediative valuation $v_g$ assigning to each atomic proposition $p$ either a type-1 or a type-2 mediative truth value:
\begin{itemize}
  \item in the type-1 case,
  \[
    v_g(p) = (\mu_{p,g}, \nu_{p,g}) \in [0,1]^2;
  \]
  \item in the type-2 case,
  \[
    v_g(p) = (\tilde{\mu}_{p,g}, \tilde{\nu}_{p,g}),
  \]
  where $\tilde{\mu}_{p,g}$ and $\tilde{\nu}_{p,g}$ are interval type-2 fuzzy sets on $[0,1]$ with
  footprints of uncertainty $\mathrm{FOU}(\tilde{\mu}_{p,g})$ and $\mathrm{FOU}(\tilde{\nu}_{p,g})$,
  as in Section~\ref{sec:MFL2}.
\end{itemize}

A type-3 mediative truth assignment for $p$ is the indexed family
\[
  v^{(3)}(p) := (v_g(p))_{g \in G}.
\]
More generally, a type-3 valuation for formulas is
\[
  v^{(3)}(\varphi) := (v_g(\varphi))_{g \in G},
\]
where each $v_g$ extends to compound formulas by the same compositional clauses as in MFL-T1 and
MFL-T2. In particular, all connectives (conjunction, disjunction, negation, implication, and the
mediative connective) are evaluated \emph{locally} at each granule, yielding a family of local
truth values prior to any cross-granule aggregation.

Here ``type-3'' is used in the sense of a higher-level granular family of mediative evaluations,
i.e., a granule-indexed valuation, without committing to any particular formalization of
type-3 fuzzy sets beyond what is needed for the present semantics.

\begin{definition}[Local and granular mediative evaluations]
\label{def:MFL3-evals}
Fix a finite set of granules $G$. For each granule $g \in G$ and formula $\varphi$, define a
\emph{local} scalar mediative degree $M_g(\varphi)$ as follows:
\begin{itemize}
  \item \emph{Type-1 local evaluation.} If $v_g(\varphi)=(\mu_{\varphi,g},\nu_{\varphi,g})$, set
  \[
    M_g(\varphi):=M(\mu_{\varphi,g},\nu_{\varphi,g}).
  \]

  \item \emph{Type-2 local evaluation (via type-reduction).} If
  $v_g(\varphi)=(\tilde{\mu}_{\varphi,g},\tilde{\nu}_{\varphi,g})$, obtain type-reduced intervals
  \[
    \mathrm{TR}(\tilde{\mu}_{\varphi,g})=[\mu_{\varphi,g}^l,\mu_{\varphi,g}^r],\qquad
    \mathrm{TR}(\tilde{\nu}_{\varphi,g})=[\nu_{\varphi,g}^l,\nu_{\varphi,g}^r],
  \]
  via a centroid type-reduction method (e.g., Karnik--Mendel procedures
  \cite{KM2001,KM1999,Mendel2001}). Choose a crisp representative of each interval; for definiteness,
  we use the midpoint
  \[
    \bar{\mu}_{\varphi,g}:=\frac{\mu_{\varphi,g}^l+\mu_{\varphi,g}^r}{2},
    \qquad
    \bar{\nu}_{\varphi,g}:=\frac{\nu_{\varphi,g}^l+\nu_{\varphi,g}^r}{2},
  \]
  and set
  \[
    M_g(\varphi):=M(\bar{\mu}_{\varphi,g},\bar{\nu}_{\varphi,g}).
  \]
\end{itemize}
If one prefers to preserve uncertainty at the local level, one may instead propagate bounds and report an interval $M_g(\varphi)\in[M_{g,L}(\varphi),M_{g,U}(\varphi)]$, as in Definition~\ref{def:type2-med-eval}.

A \emph{granular aggregation operator} for $\varphi$ is a mapping
\[
  A_\varphi : [0,1]^G \to [0,1],
  \qquad
  (M_g(\varphi))_{g\in G} \longmapsto M_G(\varphi),
\]
and the resulting group-level (global) mediative degree is
\[
  M_G(\varphi) \;:=\; A_\varphi\bigl((M_g(\varphi))_{g\in G}\bigr).
\]
Typical choices for $A_\varphi$ include (i) weighted averages with weights reflecting the reliability
or relevance of each granule; (ii) ordered weighted averaging (OWA) operators; and (iii) hierarchical
combinations respecting predefined groupings (e.g., expert groups, sensor classes, or time windows).
\end{definition}

The granular setting becomes meaningful precisely when the family $(M_g(\varphi))_{g\in G}$ is
heterogeneous, since then $A_\varphi$ can encode domain policies (e.g., robustness to outliers or
priority of trusted sources) that are not captured by a single unindexed mediative evaluation.

\begin{theorem}[Consistency under homogeneous granules and reduction to lower types]
\label{thm:MFL3-homogeneous}
Assume that for each formula $\varphi$, the aggregation operator $A_\varphi$ is idempotent, i.e.,
\[
  A_\varphi\bigl((c)_{g\in G}\bigr)=c
  \quad\text{for every constant family } (c)_{g\in G}.
\]
If all granules assign the same local mediative degree $M_g(\varphi)=c$ to $\varphi$, then the
global mediative degree satisfies $M_G(\varphi)=c$.

Moreover, fix an atomic proposition $p$ and assume that the local truth values for $p$ are identical
across granules, i.e., $v_g(p)=v_{g'}(p)$ for all $g,g'\in G$. Then, for every formula $\varphi$ in
the sublanguage generated by $\{p\}$, the type-3 assignment $v^{(3)}(\varphi)$ reduces, up to this
common value, to a lower-type assignment: define $v^{(2)}(p):=v_g(p)$ (for any $g\in G$) when the
common value is type-2, and identify $v^{(1)}(p):=v_g(p)$ when the common value is type-1. In this
homogeneous case, MFL-T3 reduces to MFL-T2 (or further to MFL-T1).
\end{theorem}

\begin{proof}
If $M_g(\varphi)=c$ for all $g\in G$, then $(M_g(\varphi))_{g\in G}=(c)_{g\in G}$ and, by
idempotence, $M_G(\varphi)=A_\varphi((c)_{g\in G})=c$.

For the reduction claim, assume that $v_g(p)=v_{g'}(p)$ for all $g,g'\in G$. Since each local
valuation $v_g$ extends to compound formulas by the same compositional clauses as in MFL-T1/MFL-T2,
it follows by structural induction that $v_g(\varphi)=v_{g'}(\varphi)$ for every formula $\varphi$
built from the single atom $p$. Hence $M_g(\varphi)$ is constant in $g$, and the first part yields
$M_G(\varphi)=M_g(\varphi)$.

Identifying the common local value as a lower-type assignment (type-2 or type-1, depending on the
case), therefore, yields a semantics that is indistinguishable from the type-3 one in homogeneous
situations.
\end{proof}

MFL-T3 is particularly suitable in domains where evidence is both heterogeneous and dynamically evolving, such as longitudinal medical diagnosis involving multiple specialists, smart-grid monitoring with diverse sensing infrastructures, or ensembles of machine-learning models deployed alongside human oversight.

\paragraph{Axioms and proof system}
The axiomatic system introduced in Section~\ref{sec:prop-MFL} is kept unchanged. In MFL-T3 we
enrich only the semantic domain by indexing valuations over a granular set $G$ and by adding an
explicit granular aggregation stage; syntactic derivability is defined exactly as in the type-1
logic.

\section{Quantum Mediative Fuzzy Logic (QMFL)}
\label{sec:QMFL}
Quantum Mediative Fuzzy Logic (QMFL) extends mediative truth values to a quantum setting in which evidence is represented by quantum states and effects on a Hilbert space. Algebraically, the framework fits naturally within effect-algebraic semantics in quantum theory~\cite{FoulisBennett1994}. From a granular viewpoint, QMFL can also be seen as an instance of effect-based granular computing: quantum effects serve as granules, and their combinations implement mediative aggregation.

Let $\mathcal{H}$ be a finite-dimensional Hilbert space. The central semantic objects are quantum effects, following the effect-based granulation viewpoint~\cite{ross2025foundationsquantumgranularcomputing}.

A (quantum) \emph{effect} on $\mathcal{H}$ is a self-adjoint operator $E$ such that
\begin{equation}\label{eq:effect-def}
  0 \preceq E \preceq I,
\end{equation}
where $\preceq$ denotes the L\"owner (positive semidefinite) order. A \emph{quantum state} is a density operator $\rho$ on $\mathcal{H}$, i.e., $\rho \succeq 0$ and $\operatorname{Tr}(\rho)=1$. The Born expectation $\operatorname{Tr}(\rho E)$ is interpreted as a graded degree associated with the effect $E$.

\begin{definition}[Quantum mediative structure]
\label{def:qmfl-structure}
For each proposition $p$, consider two effects $E_p^{+}$ and $E_p^{-}$ on $\mathcal{H}$, representing a \emph{positive channel} (evidence in favour of $p$) and a \emph{negative channel} (evidence in favour of $\neg p$), respectively. For a quantum state $\rho$, define
\[
  \mu_p(\rho) := \operatorname{Tr}\!\bigl(\rho E_p^{+}\bigr),
  \qquad
  \nu_p(\rho) := \operatorname{Tr}\!\bigl(\rho E_p^{-}\bigr),
\]
and set the associated hesitation and contradiction degrees as
\[
  \pi_p(\rho) := \max\{0,\, 1 - \mu_p(\rho) - \nu_p(\rho)\},
  \qquad
  \zeta_p(\rho) := \max\{0,\, \mu_p(\rho) + \nu_p(\rho) - 1\}.
\]
The pair $(\mu_p(\rho),\nu_p(\rho))$ is the (type-1) mediative truth value of $p$ in the quantum state $\rho$.
\end{definition}

The effects $E_p^{+}$ and $E_p^{-}$ are not required to be complementary (e.g., $E_p^{-}\neq I-E_p^{+}$). They encode two independent channels of graded evidence; accordingly, semantic contradiction is admissible.

\begin{definition}[Quantum mediative effect and degree]
\label{def:qmfl-effect}
Define the mediative weights
\begin{equation}
\label{eq:qmfl-weights}
  w_{1,p}(\rho) := 1 - \pi_p(\rho) - \frac{\zeta_p(\rho)}{2},
  \qquad
  w_{2,p}(\rho) := \pi_p(\rho) + \frac{\zeta_p(\rho)}{2},
\end{equation}
where $w_{1,p}(\rho),w_{2,p}(\rho)\in[0,1]$ and $w_{1,p}(\rho)+w_{2,p}(\rho)=1$.

\medskip

The \emph{quantum mediative effect} associated with $p$ in the state $\rho$ is
\begin{equation}
\label{eq:qmfl-Mp}
  M_p(\rho) := w_{1,p}(\rho)\,E_p^{+} + w_{2,p}(\rho)\,\bigl(I - E_p^{-}\bigr).
\end{equation}

\medskip

The corresponding \emph{quantum mediative degree} of $p$ in the state $\rho$ is
\begin{equation}
\label{eq:qmfl-Mq}
  M_q(p,\rho) := \operatorname{Tr}\!\bigl(\rho\, M_p(\rho)\bigr).
\end{equation}
\end{definition}

The operator $M_p(\rho)$ depends on $\rho$ through the weights $w_{1,p}(\rho)$ and $w_{2,p}(\rho)$. Operationally, this can be read as an adaptive, evidence-conditioned effect constructed from the current informational state $\rho$; the scalar $M_q(p,\rho)$ in~\eqref{eq:qmfl-Mq} is the corresponding Born expectation and yields a graded mediative score in $[0,1]$.

\begin{theorem}[Effect-algebra compatibility]
\label{thm:qmfl-effect-compat}
For every proposition $p$ and every quantum state $\rho$, the operator $M_p(\rho)$ is an effect, i.e., it satisfies~\eqref{eq:effect-def} with $E=M_p(\rho)$.
\end{theorem}

\begin{proof}
Since $E_p^{+}$ and $E_p^{-}$ are effects, we have $0 \preceq E_p^{+} \preceq I$ and $0 \preceq E_p^{-} \preceq I$. Hence $0 \preceq I-E_p^{-} \preceq I$, so both $E_p^{+}$ and $I-E_p^{-}$ are effects. Moreover, $w_{1,p}(\rho),w_{2,p}(\rho)\in[0,1]$ and $w_{1,p}(\rho)+w_{2,p}(\rho)=1$, so $M_p(\rho)$ is a convex combination of two effects (cf.~\eqref{eq:qmfl-Mp}). Therefore $0 \preceq M_p(\rho) \preceq I$.
\end{proof}

\begin{theorem}[Consistency with the classical mediative evaluation]
\label{thm:qmfl-consistency}
For every proposition $p$ and every quantum state $\rho$,
\[
  M_q(p,\rho) \;=\; M\!\bigl(\mu_p(\rho),\nu_p(\rho)\bigr),
\]
where $M(\mu,\nu)$ denotes the type-1 mediative evaluation defined by the mediative operator.
\end{theorem}

\begin{proof}
By Definition~\ref{def:qmfl-effect},
\[
  M_q(p,\rho)
  = \operatorname{Tr}\!\bigl(\rho\,M_p(\rho)\bigr)
  = \operatorname{Tr}\!\Bigl(\rho\bigl[w_{1,p}(\rho)E_p^{+}+w_{2,p}(\rho)(I-E_p^{-})\bigr]\Bigr).
\]
By linearity of the trace and $\operatorname{Tr}(\rho)=1$,
\[
  M_q(p,\rho)
  = w_{1,p}(\rho)\operatorname{Tr}\!\bigl(\rho E_p^{+}\bigr)
    + w_{2,p}(\rho)\bigl(1-\operatorname{Tr}(\rho E_p^{-})\bigr)
  = w_{1,p}(\rho)\,\mu_p(\rho) + w_{2,p}(\rho)\,\bigl(1-\nu_p(\rho)\bigr).
\]
By definition of the type-1 mediative evaluation, the right-hand side is exactly $M(\mu,\nu)$ evaluated at $\mu=\mu_p(\rho)$ and $\nu=\nu_p(\rho)$.
\end{proof}

\paragraph{Classical reduction and absence of coherences}
If $\rho$, $E_p^{+}$ and $E_p^{-}$ commute (hence are simultaneously diagonalizable), then $\mu_p(\rho)$ and $\nu_p(\rho)$ depend only on the diagonal entries in a common eigenbasis. In this sense, QMFL reduces operationally to the classical mediative semantics when quantum coherences play no operational role.

\paragraph{Minimal qubit-level instantiation}
A simple instantiation consistent with Definitions~\ref{def:qmfl-structure} and~\ref{def:qmfl-effect} can be given on $\mathcal{H}=\mathbb{C}^2$ by fixing $\rho = |0\rangle\langle 0|$ and choosing diagonal effects
\[
  E_{p}^{+} =
  \begin{pmatrix}
    \mu & 0\\
    0 & 0
  \end{pmatrix},
  \qquad
  E_{p}^{-} =
  \begin{pmatrix}
    \nu & 0\\
    0 & 0
  \end{pmatrix},
  \qquad
  \mu,\nu\in[0,1].
\]
Then $\mu_p(\rho)=\mu$ and $\nu_p(\rho)=\nu$, and Theorem~\ref{thm:qmfl-consistency} yields
\[
  M_q(p,\rho)=M(\mu,\nu).
\]
A concrete numerical instantiation for the obstacle detection case study is given in Section~\ref{sec:example}.

\paragraph{Finite-shot estimation and safety margins}
In an actual quantum implementation, $M_q(p,\rho)=\operatorname{Tr}(\rho\,M_p(\rho))$ must be estimated from a finite number of measurement shots. Let $X_1,\dots,X_N\in[0,1]$ denote the observed outcomes of an estimation procedure for $\operatorname{Tr}(\rho\,M_p(\rho))$, and define
\[
  \widehat{M}_q := \frac{1}{N}\sum_{k=1}^N X_k.
\]
Standard concentration bounds (e.g., Hoeffding's inequality) imply that $\widehat{M}_q$ concentrates around $M_q(p,\rho)$ as $N$ increases. Consequently, in safety-critical settings, decision rules based on $\widehat{M}_q$ should incorporate an explicit uncertainty margin so that finite-shot fluctuations cannot systematically overturn conservative actions in the presence of strong (possibly conflicting) evidence.

\section{Illustrative example: safety-first mediative sensor fusion in obstacle detection}
\label{sec:example}
We present a case study illustrating the safety-first interpretation of mediative truth degrees in a domain beyond the original control applications of MFL. We consider a simple obstacle-detection scenario in autonomous driving. This example shows that, in a conservative regime with low higher-order uncertainty, all mediative variants yield the same control decision, while still allowing a structured path to richer behaviours when uncertainty, granularity, or quantum effects matter. Throughout, we use a strict safety-first design: decision thresholds and aggregation weights are chosen so that strong evidence of danger—even if conflicting—favours braking or deceleration over proceeding as if no obstacle were present.

Let $p$ be the proposition
\[
  p : \text{``There is a dangerous obstacle within 20 meters in front of the vehicle.''}
\]
Assume that two independent perception channels are available: a radar/LiDAR channel, which is robust under bad weather but has limited resolution, and a camera-based detector, which is more precise under good visibility but sensitive to glare and occlusion. Each channel provides an assessment of $p$ as a mediative truth value $(\mu,\nu)$, where $\mu$ is the degree of support for $p$ and $\nu$ is the degree of support for $\neg p$.

\subsection{Safety-first aggregation and decision thresholds}
\label{sec:example-aggregation}

\begin{figure}[t]
  \centering
  \resizebox{1.0\textwidth}{!}{%
    \usetikzlibrary{calc}
\begin{tikzpicture}[
  font=\small,
  >=Latex,
  sensor/.style={rectangle,draw,rounded corners,align=center,inner sep=3pt},
  block/.style={rectangle,draw,rounded corners,align=center,inner sep=4pt,fill=blue!5},
  qblock/.style={rectangle,draw,rounded corners,align=center,inner sep=4pt,fill=purple!5},
  decision/.style={rectangle,draw,rounded corners,align=center,inner sep=4pt,fill=green!5},
  arrow/.style={->,thick},
  agg/.style={circle,draw,inner sep=1.5pt,fill=gray!10},
  every node/.style={align=center}
]

\node[font=\bfseries,anchor=west,align=center] (lefttitle)
  at (0,0)
  {Heterogeneous, uncertain, \\%
  possibly conflicting\\%
    evidence};

\node[sensor,anchor=west] (cam)
  at ([yshift=-1.5cm]lefttitle.west) {Camera};

\node[sensor,anchor=west] (radar)
  at ([yshift=-1.5cm]cam.west) {Radar / LiDAR};

\node[sensor,anchor=west] (weather)
  at ([yshift=-1.5cm]radar.west) {Weather\\sensor};

\node[sensor,anchor=west] (experts)
  at ([yshift=-1.5cm]weather.west) {Human\\experts};

\node[anchor=west,font=\scriptsize,inner sep=1pt] at ($(cam.east)+(0.1,0.15)$) {$\uparrow$ safe};
\node[anchor=west,font=\scriptsize,inner sep=1pt] at ($(cam.east)+(0.1,-0.15)$) {$\uparrow$ danger};

\node[anchor=west,font=\scriptsize,inner sep=1pt] at ($(radar.east)+(0.1,0.15)$) {$\uparrow$ safe};
\node[anchor=west,font=\scriptsize,inner sep=1pt] at ($(radar.east)+(0.1,-0.15)$) {$\uparrow$ danger};

\node[anchor=west,font=\scriptsize,inner sep=1pt] at ($(weather.east)+(0.1,0.15)$) {$\uparrow$ safe};
\node[anchor=west,font=\scriptsize,inner sep=1pt] at ($(weather.east)+(0.1,-0.15)$) {$\uparrow$ danger};

\node[anchor=west,font=\scriptsize,inner sep=1pt] at ($(experts.east)+(0.1,0.15)$) {$\uparrow$ safe};
\node[anchor=west,font=\scriptsize,inner sep=1pt] at ($(experts.east)+(0.1,-0.15)$) {$\uparrow$ danger};

\node[
  block,
  right=2.8cm of radar,
  minimum width=6cm,
  text width=6cm,
  align=center
] (extract)
  {Mediative encoding of\\[1pt]
   evidence into\\[1pt]
   mediative truth--falsity\\[1pt]
   pairs $(\mu,\nu)$};


\coordinate (bus) at ($(radar.east)+(1.7,0)$);

\draw[thick] (bus |- cam.east) -- (bus |- experts.east);

\draw[thick] (cam.east)     -- (bus |- cam.east);
\draw[thick] (radar.east)   -- (bus |- radar.east);
\draw[thick] (weather.east) -- (bus |- weather.east);
\draw[thick] (experts.east) -- (bus |- experts.east);

\draw[arrow] (bus |- radar.east) -- (extract.west);


\node[font=\scriptsize,anchor=west,align=left] (legend)
  at ([yshift=-1.3cm]experts.west)
  {Legend (for proposition $p$): $\uparrow$ danger = degree $\mu$ (support for $p$), \\$\uparrow$ safe = degree $\nu$ (support for $\neg p$); equivalently $b=1-\nu$ measures \\lack of support for safety.
};


\node[
  block,
  right=2.0cm of extract,
  minimum width=5.2cm,
  minimum height=1.2cm,
  yshift=1.66cm,] (t1) {
  Type-1 Mediative Fuzzy Logic (MFL-T1)\\[-1pt]
  {\scriptsize Mediative operator $M(\mu,\nu)$}
};

\node[
  block,
  below=0.45cm of t1,
  minimum width=5.2cm,
  minimum height=1.2cm,
  align=center
] (t2) {
  Type-2 Mediative Fuzzy Logic (MFL-T2)\\[-1pt]
  {\scriptsize Footprints of uncertainty for $\mu$ and $\nu$}
};

\node[
  block,
  below=0.45cm of t2,
  minimum width=5.2cm,
  minimum height=1.2cm,
  align=center
] (t3) {
  Type-3 granular MFL (MFL-T3)\\[-1pt]
  {\scriptsize Granular family $v^{(3)}(p)$ over expert--sensor--time granules}
};



\node[
  qblock,
  below=0.45cm of t3,
  minimum width=6.8cm,   
  align=center
] (qmfl)
  {Quantum Mediative Fuzzy Logic (QMFL)\\[1pt]
   Hilbert-space implementation of \\[1pt]
   mediative operator $M$ with quantum interference};

\draw[arrow] (extract.east) -- ++(0.4,0) |- (t1.west);
\draw[arrow] (extract.east) -- ++(0.4,0) |- (t2.west);
\draw[arrow] (extract.east) -- ++(0.4,0) |- (t3.west);
\draw[arrow] (extract.east) -- ++(0.4,0) |- (qmfl.west);


\node[
  decision,
  right=3.1cm of t2,
  yshift=-0.0cm, 
  minimum width=2.3cm,
  minimum height=3.0cm
] (dec) {};

\node[font=\bfseries,align=center] (variantsTitle)
  at ($ (t1.center |- lefttitle.north) + (0,-0.70) $)
  {MFL variants\\[1pt]\scriptsize Alternative MFL-based reasoning schemes \\ [1pt]\scriptsize(one active at a time)};

\draw ($(dec.north west)+(0.25,-0.35)$) rectangle ($(dec.north east)+(-0.25,-1.15)$);
\draw ($(dec.north west)+(0.25,-1.15)$) rectangle ($(dec.north east)+(-0.25,-1.95)$);
\draw ($(dec.north west)+(0.25,-1.95)$) rectangle ($(dec.south east)+(-0.25,0.35)$);

\node[font=\scriptsize] at ($(dec.north west)!0.5!(dec.north east)+(0,-0.75)$) {Emergency\\brake};
\node[font=\scriptsize] at ($(dec.north west)!0.5!(dec.north east)+(0,-1.55)$) {Cautious\\slow-down};
\node[font=\scriptsize] at ($(dec.north west)!0.5!(dec.north east)+(0,-2.35)$) {Proceed};

\coordinate (rbus) at ($(dec.west)+(-0.9,0)$);

\draw[thick] (rbus |- t1.east) -- (rbus |- qmfl.east);

\draw[thick] (t1.east)  -- (rbus |- t1.east);
\draw[thick] (t2.east)  -- (rbus |- t2.east);
\draw[thick] (t3.east)  -- (rbus |- t3.east);
\draw[thick] (qmfl.east) -- (rbus |- qmfl.east);

\draw[arrow] (rbus) -- (dec.west);

\node[font=\scriptsize,align=center,below=0.3cm of dec] (thresholdText)
  {Output degree $M(\mu,\nu)$ or $M_q$\\
   compared with safety-biased \\ thresholds to trigger \\ Brake / Slow-down / Go};

\end{tikzpicture}%
  }
  \caption{Instantiation of the mediative pipeline in the autonomous braking case study.}
  \label{fig:mfl-case-pipeline}
\end{figure}
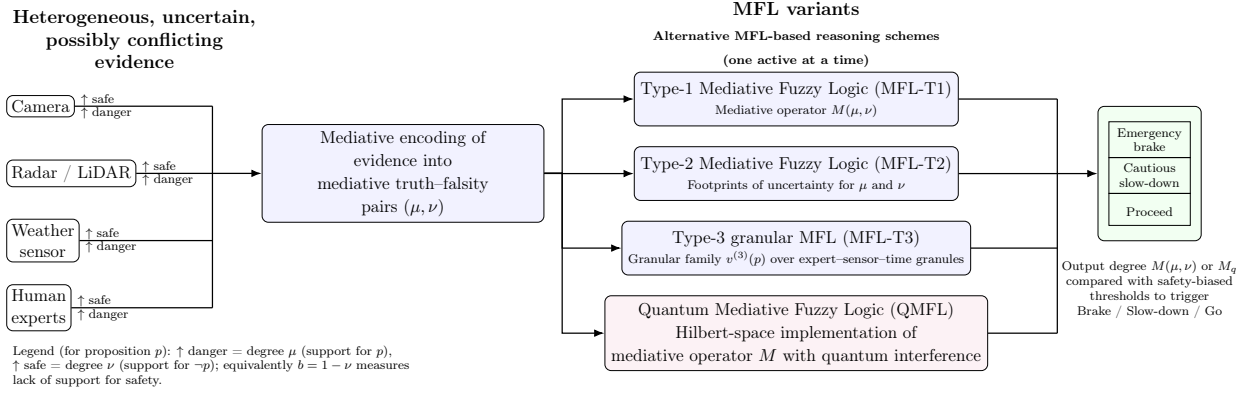

Figure~\ref{fig:mfl-case-pipeline} summarizes how the proposed mediative pipeline is instantiated in this case study. The left panel shows heterogeneous sources of evidence and their mediative encoding into truth--falsity pairs $(\mu,\nu)$. The central panel highlights the four alternative Mediative Fuzzy Logic variants (MFL-T1, MFL-T2, MFL-T3, and QMFL), of which only one is active at a time. The right panel depicts the safety-first decision module, where the mediative degree $M(\mu,\nu)$ or $M_q$ is compared against safety-biased thresholds to trigger Emergency brake, Cautious slow-down, or Proceed.

Although the pipeline in Fig.~\ref{fig:mfl-case-pipeline} includes four potential evidence sources (camera, radar/LiDAR, weather sensor, and human experts), in the example below we use only two channels (camera and radar/LiDAR) to keep the notation and plots readable. The mediative encoding and aggregation extend directly to additional sources using the same scheme.

To maintain a strict safety-first approach, we specify (i) how the two channels are aggregated and (ii) how the resulting mediative value maps to control actions. We use a weighted aggregation of radar and camera assessments:
\begin{equation}\label{eq:fusion-aggregation}
  (\mu,\nu)
  = \bigl(\,\alpha \mu_{\mathrm{radar}} + (1 - \alpha)\mu_{\mathrm{cam}},\;
          \alpha \nu_{\mathrm{radar}} + (1 - \alpha)\nu_{\mathrm{cam}}\,\bigr),
\end{equation}
where $\alpha \in [0,1]$ reflects the relative trust in the radar channel under the current context (weather, visibility, road type, etc.). In poor visibility conditions, we take $\alpha$ closer to $1$, favoring radar; in more balanced conditions, we may take $\alpha = 0.5$.

On top of the mediative evaluation $M(\mu,\nu)$, we adopt a simple decision policy based on a braking threshold $T_{\mathrm{brake}} = 0.7$:
\begin{itemize}
  \item if $M(\mu,\nu) \ge 0.7$: the vehicle should brake decisively (emergency or strong braking),
  \item if $0.5 \le M(\mu,\nu) < 0.7$: the vehicle should decelerate and seek additional sensing,
  \item if $M(\mu,\nu) < 0.5$: the vehicle may continue cautiously while monitoring for changes.
\end{itemize}

The numerical examples below are chosen so that even mild contradictions are resolved in favour of safety when the radar channel strongly suggests the presence of an obstacle.

\paragraph{Three evidence configurations}
We consider three representative configurations. In Cases~1 and~3, we assume poor visibility and set $\alpha = 0.7$ to prioritize the radar/LiDAR channel. In Case~2, we model a highly ambiguous situation (e.g., night glare) and set $\alpha = 0.5$ to reflect a deliberately balanced fusion between radar/LiDAR and camera. Table~\ref{tab:mfl-example-cases} reports the radar/LiDAR and camera assessments, the aggregated mediative pair $(\mu,\nu)$ obtained from~\eqref{eq:fusion-aggregation}, the resulting hesitation $\pi$, contradiction $\zeta$, the mediative degree $M(\mu,\nu)$, and the corresponding safety-first control action.

\begin{table}[t]
\centering\scriptsize
\caption{Three illustrative evidence configurations and their aggregated mediative scores
(Cases 1 and 3 use $\alpha=0.7$, Case 2 uses $\alpha=0.5$).}
\label{tab:mfl-example-cases}
\begin{tabular}{c c c c c c c c}
\hline
Case $i$ &
$(\mu_{\mathrm{radar}},\nu_{\mathrm{radar}})$ &
$(\mu_{\mathrm{cam}},\nu_{\mathrm{cam}})$ &
$(\mu^{(i)},\nu^{(i)})$ &
$\pi^{(i)}$ &
$\zeta^{(i)}$ &
$M(\mu^{(i)},\nu^{(i)})$ &
Action \\
\hline
1 & (0.80, 0.10) & (0.40, 0.20) & (0.68, 0.13) & 0.19 & 0.00 & $\approx 0.716$ & brake decisively \\
2 & (0.90, 0.10) & (0.10, 0.90) & (0.50, 0.50) & 0.00 & 0.00 & $0.5$           & decelerate, seek more sensing \\
3 & (0.95, 0.05) & (0.20, 0.90) & (0.725, 0.305) & 0.00 & 0.03 & $\approx 0.724$ & brake decisively \\
\hline
\end{tabular}
\end{table}

Under the safety-first thresholds adopted above, configurations 1 and 3 trigger decisive braking ($M(\mu,\nu)\ge 0.7$), while configuration 2 triggers cautious slow-down ($0.5 \le M(\mu,\nu) < 0.7$). Note that configuration 3 exhibits explicit contradiction ($\zeta^{(3)}>0$), modelling the case where one channel provides strong counter-evidence while the other still suggests danger.

\subsection{Case 1: incomplete but non-contradictory information}
\label{sec:example-case1}

Suppose that under light fog, the radar detects a strong reflection consistent with an obstacle, while the camera has low confidence because the image is partially blurred. We model this as
\[
  (\mu_{\mathrm{radar}}, \nu_{\mathrm{radar}}) = (0.8, 0.1), \qquad
  (\mu_{\mathrm{cam}},   \nu_{\mathrm{cam}})   = (0.4, 0.2).
\]
Since fog favors the radar channel, we choose $\alpha = 0.7$. The aggregated mediative value for $p$ becomes
\[
  (\mu,\nu)
  = \bigl(0.7 \cdot 0.8 + 0.3 \cdot 0.4,\; 0.7 \cdot 0.1 + 0.3 \cdot 0.2\bigr)
  = (0.68, 0.13).
\]
Here $\mu + \nu = 0.81 < 1$, so the situation is one of incomplete information with hesitation
\[
  \pi(\mu,\nu) = 1 - \mu - \nu = 0.19, \qquad \zeta(\mu,\nu) = 0.
\]
Using the mediative evaluation
\[
  M(\mu,\nu) = \mathcal{M}(a,b;\pi,\zeta), \qquad a = \mu,\; b = 1 - \nu,
\]
and recalling that when $\zeta = 0$ the operator reduces to the intuitionistic combination, we obtain
\[
  a = 0.68, \qquad b = 1 - 0.13 = 0.87,
\]
\[
  M(\mu,\nu) = (1 - \pi)a + \pi b
    \approx 0.81 \cdot 0.68 + 0.19 \cdot 0.87 \approx 0.716.
\]
Thus, despite the hesitation, the mediative truth degree of $p$ is above the braking threshold $T_{\mathrm{brake}} = 0.7$. This aligns with a safety-first design: under poor visibility, strong radar evidence for an obstacle is enough to trigger decisive braking, even if the camera is uncertain.

\subsection{Case 2: strong but symmetric conflict}
\label{sec:example-case2}

Consider now a situation at night with glare: the radar still reports a strong reflection consistent with an obstacle, but the camera confidently ``sees'' an empty road ahead. We model this as
\[
  (\mu_{\mathrm{radar}}, \nu_{\mathrm{radar}}) = (0.9, 0.1), \qquad
  (\mu_{\mathrm{cam}},   \nu_{\mathrm{cam}})   = (0.1, 0.9).
\]
To reflect an unbiased fusion in this particularly ambiguous situation, we take $\alpha = 0.5$, obtaining
\[
  (\mu,\nu)
  = \bigl(0.5 \cdot 0.9 + 0.5 \cdot 0.1,\; 0.5 \cdot 0.1 + 0.5 \cdot 0.9\bigr)
  = (0.5, 0.5).
\]
In this case $\mu + \nu = 1$, so both hesitation and contradiction vanish:
\[
  \pi(\mu,\nu) = \max(0, 1 - \mu - \nu) = 0, \qquad
  \zeta(\mu,\nu) = \max(0, \mu + \nu - 1) = 0.
\]
Semantically, the available evidence is perfectly balanced between $p$ and $\neg p$. The mediative evaluation is then
\[
  a = \mu = 0.5, \qquad b = 1 - \nu = 0.5,
\]
\[
  M(\mu,\nu) = \mathcal{M}(a, b; 0,0) = a = 0.5.
\]
The resulting mediative truth degree is exactly intermediate. According to the policy above, the vehicle should decelerate and seek additional sensing (for example, by re-scanning, adjusting exposure, or relying more on other sensors), but an immediate emergency stop is not required. This configuration shows how MFL can represent a genuinely undecided state without forcing full braking or full trust in one channel.

\subsection{Case 3: explicit overdetermined contradiction with safety-first resolution}
\label{sec:example-case3}

To obtain a genuinely contradictory configuration, suppose instead that the vehicle receives conservative estimates in which each channel strongly defends its own assessment. For instance, the radar is almost certain that an obstacle is present, while the camera almost certainly indicates no obstacle due to a misleading reflection or occlusion:
\[
  (\mu_{\mathrm{radar}}, \nu_{\mathrm{radar}}) = (0.95, 0.05), \qquad
  (\mu_{\mathrm{cam}},   \nu_{\mathrm{cam}})   = (0.2,  0.9).
\]
In this configuration, we again take $\alpha = 0.7$ to reflect a design choice that, under severe conflict, favours the safer channel (radar) without completely discarding the camera. The aggregated value is
\[
  (\mu,\nu)
  = \bigl(0.7 \cdot 0.95 + 0.3 \cdot 0.2,\; 0.7 \cdot 0.05 + 0.3 \cdot 0.9\bigr)
  = (0.725, 0.305).
\]
Now $\mu + \nu \approx 1.03 > 1$, so we have overdetermined, potentially contradictory information, with
\[
  \pi(\mu,\nu) = 0, \qquad
  \zeta(\mu,\nu) = \mu + \nu - 1 \approx 0.03.
\]
The mediative evaluation becomes
\[
  a = \mu = 0.725, \qquad b = 1 - \nu = 0.695,
\]
\[
  w_1 = 1 - \pi - \zeta/2 \approx 0.985, \qquad
  w_2 = \pi + \zeta/2 \approx 0.015,
\]
\[
  M(\mu,\nu) = w_1 a + w_2 b
    \approx 0.985 \cdot 0.725 + 0.015 \cdot 0.695 \approx 0.724.
\]
Here, the mediative truth degree is slightly above $0.72$, comfortably above the braking threshold $T_{\mathrm{brake}}$. Even with explicit contradiction ($\zeta > 0$), the system resolves the conflict in favor of safety because the more trusted radar channel dominates. This shows that even mild contradictions should push $M(\mu,\nu)$ into a braking regime when at least one reliable sensor strongly supports a hazard.

\paragraph{QMFL instantiation for the three evidence configurations}
To connect QMFL with Table~\ref{tab:mfl-example-cases}, we fix $\mathcal{H}=\mathbb{C}^2$ and $\rho=|0\rangle\langle 0|$. For each configuration $i$, we encode the aggregated classical degrees $(\mu^{(i)},\nu^{(i)})$ as diagonal effects
\begin{equation}
\label{eq:qmfl-diagonal-encoding}
  E_{p,i}^{+} =
  \begin{pmatrix}
    \mu^{(i)} & 0\\
    0 & 0
  \end{pmatrix},
  \qquad
  E_{p,i}^{-} =
  \begin{pmatrix}
    \nu^{(i)} & 0\\
    0 & 0
  \end{pmatrix}.
\end{equation}

Let $M_{p,i}(\rho)$ be constructed from $(E_{p,i}^{+},E_{p,i}^{-})$ via Definition~\ref{def:qmfl-effect}, and define the corresponding QMFL score as $M_q^{(i)}(p):=\operatorname{Tr}\!\bigl(\rho\,M_{p,i}(\rho)\bigr)$. Since $\mu_p(\rho)=\mu^{(i)}$ and $\nu_p(\rho)=\nu^{(i)}$ under this encoding, Theorem~\ref{thm:qmfl-consistency} yields $M_q^{(i)}(p)=M(\mu^{(i)},\nu^{(i)})$. Therefore, in this conservative regime (low higher-order uncertainty and a minimal diagonal encoding), QMFL reproduces the same numerical mediative scores as MFL-T1 on the aggregated inputs, while the operational difference is that $M_q^{(i)}(p)$ would be estimated from finite-shot measurements in a concrete quantum implementation.

\subsection{Comparison across MFL variants}
\label{sec:example-comparison-variants}

The obstacle-detection example can also be interpreted within MFL-T2, MFL-T3, and QMFL. Under natural modelling assumptions (interval type-2 footprints centred around the type-1 values, homogeneous granules for the two sensors, and diagonal quantum effects in a classical limit), these variants induce the same scalar mediative degrees as type-1 MFL in Cases~1--3 and therefore the same control actions. This coincidence is a direct consequence of how the higher-type semantics are instantiated in this simple scenario.

For MFL-T2, we use interval type-2 footprints whose type-reduction returns exactly the same crisp degrees $(\mu_p,\nu_p)$ as in the type-1 case, so the mediative operator sees the same input pair. For MFL-T3, all granules associated with a given proposition share the same local mediative truth value, and the granular aggregation operator is idempotent; hence, the aggregated value coincides with that common local one. For QMFL, we restrict attention to diagonal effects and density matrices whose Born expectations reproduce the same degrees of truth and falsity as in the fuzzy semantics. In other words, the case study is evaluated in a conservative fragment in which second-order uncertainty, granular heterogeneity and quantum coherences are weak or absent, so the higher-type and quantum variants collapse to the same scalar decision values as type-1 MFL. In more complex scenarios, where footprints of uncertainty (FOUs), granular structure, or non-commuting effects play a substantive role, the four frameworks may diverge both semantically and numerically. The corresponding scalar mediative evaluations for Cases~1--3, based on the aggregated inputs in Table~\ref{tab:mfl-example-cases}, are collected in Table~\ref{tab:MFL-variants}.

\begin{table}[t]
  \centering\scriptsize
  \caption{Scalar mediative evaluations across MFL variants for the obstacle detection example.}
  \label{tab:MFL-variants}
  \begin{tabular}{lccc}
    \toprule
    Framework & Case 1 & Case 2 & Case 3 \\
    \midrule
    MFL (type-1)
      & $M(\mu,\nu) \approx 0.716$
      & $M(\mu,\nu) = 0.5$
      & $M(\mu,\nu) \approx 0.724$ \\
    MFL-T2
      & $\bar{M}_p \approx 0.716$
      & $\bar{M}_p \approx 0.5$
      & $\bar{M}_p \approx 0.724$ \\
    MFL-T3
      & $M_G(p) \approx 0.716$
      & $M_G(p) \approx 0.5$
      & $M_G(p) \approx 0.724$ \\
    QMFL
      & $M_q^{(1)}(p) \approx 0.716$
      & $M_q^{(2)}(p) \approx 0.5$
      & $M_q^{(3)}(p) \approx 0.724$ \\
    \bottomrule
  \end{tabular}
\end{table}

Beyond the numerical analysis above, the obstacle-detection scenario admits complementary readings in terms of MFL-T2, MFL-T3, and QMFL. These perspectives highlight how second-order uncertainty, granular structure, and quantum effects can be incorporated into mediative reasoning while preserving the safety-first philosophy.

\subsection{Type-2, type-3 and quantum mediative perspectives on the example}
\label{sec:example-extended}

The previous example was presented at the type-1 level, with crisp mediative truth values $(\mu,\nu)$ for each fused assessment of the obstacle proposition $p$. We now outline how the same scenario can be treated in MFL-T2, MFL-T3, and QMFL, and which design considerations are specific to each extension.

\subsubsection{Type-2 mediative extension (MFL-T2)}
\label{sec:example-MFL2}

In MFL-T2, the mediative truth value of $p$ is no longer a single pair $(\mu,\nu)$ but a pair of type-2 fuzzy sets $(\tilde{\mu}_p,\tilde{\nu}_p)$ on $[0,1]$. Intuitively, this represents second-order uncertainty about the degrees of agreement and non-agreement, for instance, due to sensor noise, calibration errors, or changing environmental conditions.

Under natural modelling assumptions (interval type-2 footprints centred around the type-1 values), the type-2 extension yields the same scalar mediative degrees as type-1 MFL in Cases~1--3 and therefore the same control actions. Concretely, after aggregating the radar and camera outputs at the primary level, we wrap the crisp pair into an interval type-2 pair $(\tilde{\mu}_p,\tilde{\nu}_p)$, from which type-2 hesitation and contradiction degrees are computed and fed to a type-2 mediative operator. The resulting mediative evaluation $\tilde{M}_p$ is itself a type-2 fuzzy set (or, after type-reduction, an interval of plausible mediative values).

A safety-first design must then specify how to interpret this interval. One conservative policy is to trigger decisive braking whenever the \emph{lower} endpoint of the type-reduced mediative interval exceeds the threshold $T_{\mathrm{brake}}$, and to trigger at least cautious deceleration whenever the \emph{upper} endpoint exceeds $T_{\mathrm{brake}}$. In this way, MFL-T2 captures uncertainty about the mediative degree itself and allows the control strategy to distinguish between clearly safe, borderline and clearly dangerous configurations at a second-order level.

\subsubsection{Granular view in MFL-T3}
\label{sec:example-MFL3}

In MFL-T3, mediative truth is organized in a granular way over a family of indices. We write
\begin{equation}
\label{eq:granule-form}
  g := (\text{sensor},\ \text{time window},\ \text{context}),
\end{equation}
and let $G$ denote the set of admissible granules (i.e., all indices of the form~\eqref{eq:granule-form}), so that each granule $g \in G$ carries its own mediative truth value $v_g(p) = (\mu_g,\nu_g)$ and mediative evaluation $M_g$. In the obstacle-detection example, we may take granules of the form
\[
  g = (\text{radar or camera},\; \text{current frame or recent window},\; \text{driving context}),
\]
so that short-term history (several consecutive frames), multiple sensors (radar, cameras, possibly LiDAR) and context (weather, road type) are all represented at the granular level.

The MFL-T3 semantics then aggregates the family $\{v_g(p) : g \in G\}$ through explicit granular aggregation operators as defined in Section~\ref{sec:MFL3}. For obstacle detection, a safety-first design requires these operators to satisfy two conditions:
(i) monotonicity with respect to each granule (stronger local evidence for an obstacle cannot decrease the global mediative degree), and
(ii) a dominance property ensuring that if any sufficiently trusted granule strongly supports the presence of a dangerous obstacle, then the aggregated mediative evaluation crosses the braking threshold. At the same time, the granular structure allows temporal smoothing and cross-checking between sensors: for instance, a single spurious camera frame can be downweighted if both radar and neighbouring frames disagree with it.

\subsubsection{Quantum mediative perspective (QMFL)}
\label{sec:example-QMFL}

In QMFL, the semantics of the obstacle proposition $p$ is specified by a quantum state $\rho$ together with two effects $E_p^{+}$ and $E_p^{-}$ on a Hilbert space $\mathcal{H}$, encoding, respectively, support for $p$ and support for $\neg p$ (cf.\ Section~\ref{sec:QMFL}). The mediative degrees are then obtained as Born expectations
\[
  \mu_p(\rho) = \mathrm{Tr}(\rho E_p^{+}), \qquad \nu_p(\rho) = \mathrm{Tr}(\rho E_p^{-}),
\]
from which hesitation and contradiction are derived, and the quantum mediative degree $M_q(p,\rho)$ is computed via the quantum mediative effect $M_p(\rho)$ defined in Definition~\ref{def:qmfl-effect}.

For the obstacle-detection task, a quantum instantiation may arise in two ways. In a \emph{simulation} setting, the pair $(\mu,\nu)$ obtained from classical MFL fusion is encoded into diagonal density operators and diagonal effects, so that $M_q(p,\rho)$ reproduces the classical mediative value while enabling quantum-inspired processing (e.g., superpositions of hypotheses or quantum walks on state graphs). In a more ambitious \emph{hardware} setting, part of the perception or decision pipeline is implemented on a quantum device, and $\rho$ captures both sensor information and intrinsic quantum noise.

From a design perspective, the safety-first requirement translates into constraints on the encoding map from sensor evidence to $(\rho, E_p^{+}, E_p^{-})$ and on the decision rule based on $M_q(p,\rho)$. In particular, finite-shot estimation and device noise should not systematically push the estimated quantum mediative degree below the braking threshold in situations where classical evidence already mandates braking. This typically leads to conservative margins in the threshold and to robust encoding schemes for $\rho$ and the associated effects.

\subsection{MFL-T2, MFL-T3 and QMFL outcomes for the three cases}
\label{sec:example-T2T3Q}

The three numerical cases in Sections~\ref{sec:example-case1} to \ref{sec:example-case3} were analyzed at the type-1 level, producing crisp mediative values $M(\mu,\nu)$ and associated control decisions. We now describe how the same three scenarios can be handled in MFL-T2, MFL-T3, and QMFL, and summarize the corresponding outcomes.

\subsubsection*{Type-2 mediative degrees for the three cases (MFL-T2)}

In MFL-T2, a mediative truth value is a pair $(\tilde{\mu},\tilde{\nu})$ of type-2 fuzzy sets on $[0,1]$. For the present example, we adopt a simple interval type-2 representation by wrapping each crisp pair $(\mu,\nu)$ in a small rectangle
\[
  \mu \in [\mu^-,\mu^+], \qquad \nu \in [\nu^-,\nu^+],
\]
with case-dependent half-widths chosen to reflect the expected sensor variability. For each case, we then consider a pessimistic slice $(\mu^-,\nu^+)$ and an optimistic slice $(\mu^+,\nu^-)$, yielding diagonal-corner mediative bounds
\[
  M^- := M(\mu^-,\nu^+), \qquad
  M^+ := M(\mu^+,\nu^-).
\]
This provides an illustrative interval $[M^-,M^+]$ of plausible mediative values for $p$.

\begin{table}[t]
  \centering \tiny
  \caption{Illustrative interval type-2 mediative degrees for the three cases.}
  \label{tab:type2-mediative-cases}
  \begin{tabular}{lccccl}
    \toprule
    Case & Crisp $(\mu,\nu)$ & Interval for $(\mu,\nu)$
         & $M^-$ & $M^+$ & Type-2 decision \\
    \midrule
    Case 1 & $(0.68, 0.13)$
           & $\mu \in [0.65,0.71],\ \nu \in [0.10,0.16]$
           & $\approx 0.686$ & $\approx 0.746$
           & $[M_L,M_U]$ mostly above $0.7$ (brake) \\
    Case 2 & $(0.50, 0.50)$
           & $\mu,\nu \in [0.45,0.55]$
           & $0.45$ & $0.55$
           & $[M_L,M_U]$ entirely below $0.7$ (no emergency brake) \\
    Case 3 & $(0.725, 0.305)$
           & $\mu \in [0.695,0.755],\ \nu \in [0.275,0.335]$
           & $\approx 0.695$ & $\approx 0.755$
           & $[M_L,M_U]$ clearly above $0.7$ (brake) \\
    \bottomrule
  \end{tabular}
\end{table}

In Case~1, the interval $[M_L,M_U]$ lies mostly above the braking threshold $T_{\mathrm{brake}} = 0.7$, so a safety-first policy still recommends decisive braking despite second-order uncertainty. In Case~2, the entire interval lies well below $0.7$, confirming that emergency braking is not justified and that controlled deceleration is appropriate. In Case~3, the whole interval is above $0.7$, so the safety-first resolution of a strong conflict remains braking even under type-2 uncertainty. The precise widths of the intervals are only illustrative; what matters is that MFL-T2 can express a band of mediative values and that the decision rule can be formulated in terms of $M_L$ and $M_U$ (e.g., braking whenever $M_L \ge T_{\mathrm{brake}}$).

\subsubsection*{Granular outcomes in MFL-T3}

In MFL-T3, mediative truth is assigned at the level of granules $g \in G$, where a typical granule has the form~\eqref{eq:granule-form}.

For the obstacle-detection example, a minimal granular description for each case is obtained by taking two granules
\[
  g_{\mathrm{radar}} = (\text{radar},\ \text{current window},\ \text{context}),
  \qquad
  g_{\mathrm{cam}}   = (\text{camera},\ \text{current window},\ \text{context}),
\]
with mediative values $v_{g_{\mathrm{radar}}}(p)$ and $v_{g_{\mathrm{cam}}}(p)$ matching the radar and camera assessments used in Cases~1--3. A type-3 granular aggregation operator then combines $\{v_g(p) : g \in G\}$ into a global mediative value $M_G(p)$.

If we choose a top-level granular aggregator that, in the two-granule case, reduces exactly to the type-1 fusion scheme used in Sections~\ref{sec:example-case1}--\ref{sec:example-case3} (e.g., a weighted combination with the same parameter $\alpha$), then the global MFL-T3 outcomes coincide numerically with the type-1 mediative values:
\[
  M_G^{(1)}(p) \approx 0.716,\quad
  M_G^{(2)}(p) = 0.5,\quad
  M_G^{(3)}(p) \approx 0.724
\]

for Cases~1, 2 and~3 respectively. The advantage of the type-3 setting is not a different scalar output in this minimal configuration, but rather the ability to systematically incorporate additional granules: multiple frames, multiple cameras, different radar modes, or contexts. Safety-first design then constrains the granular aggregation operator so that strong evidence for an obstacle in any sufficiently trusted granule forces $M_G(p)$ above $T_{\mathrm{brake}}$, while still allowing temporal smoothing and cross-checking between sources.

\subsubsection*{Quantum mediative degrees for the three cases (QMFL)}

In QMFL, a proposition $p$ is evaluated in a state $\rho$ via two effects $E_p^{+}$ and $E_p^{-}$ as in Section~\ref{sec:QMFL}. A simple embedding of the type-1 mediative values for the three cases into QMFL is obtained by considering $\mathcal{H}=\mathbb{C}^2$, fixing $\rho=|0\rangle\langle 0|$, and encoding each aggregated crisp pair $(\mu^{(i)},\nu^{(i)})$ as diagonal effects defined in~\eqref{eq:qmfl-diagonal-encoding}. Let $M_{p,i}(\rho)$ be constructed from $(E_{p,i}^{+},E_{p,i}^{-})$ via Definition~\ref{def:qmfl-effect}, and define $M_q^{(i)}(p) := \mathrm{Tr}\!\bigl(\rho\,M_{p,i}(\rho)\bigr)$. Then $\mu_p(\rho)=\mu^{(i)}$ and $\nu_p(\rho)=\nu^{(i)}$; thus, Theorem~\ref{thm:qmfl-consistency} yields $M_q^{(i)}(p)=M(\mu^{(i)},\nu^{(i)})$. Hence, at the level of ideal expectation values, QMFL reproduces the same mediative scores obtained classically:
\[
  M_q^{(1)}(p) \approx 0.716,\qquad
  M_q^{(2)}(p) = 0.5,\qquad
  M_q^{(3)}(p) \approx 0.724.
\]
In a realistic quantum implementation, these expectation values would be estimated from a finite number of shots, introducing statistical fluctuations around $M_q^{(i)}(p)$. A safety-first design must therefore ensure that the encoding of sensor evidence and the chosen braking threshold are robust with respect to such fluctuations. For example, one may introduce a safety margin by triggering braking whenever the estimated $M_q^{(i)}(p)$ exceeds $T_{\mathrm{brake}}$ by more than a small tolerance, so that quantum statistical noise cannot systematically suppress a braking action that would be required by the classical mediative evidence.

\section{Conclusions}
\label{sec:conclusions}
We have presented a unified account of Mediative Fuzzy Logic, from its type-1 foundations to type-2, type-3, and quantum extensions. At the core is a mediative operator that combines positive and negative channels under explicit control by hesitation and contradiction. This provides a single semantic mechanism for reconciling incomplete and genuinely conflicting evidence while remaining compatible with standard fuzzy infrastructures.

At the type-1 level, mediative truth values are formalized as truth–falsity pairs equipped with bilattice-like operations and a scalar mediative evaluation. A propositional system (MFL-T1) with a mediative connective is introduced. The resulting logic is algebraically and axiomatically well behaved, supports paraconsistent reasoning, and reduces to standard fuzzy and intuitionistic-fuzzy logics in the expected special cases. These properties establish MFL-T1 as a principled foundation for safety-first aggregation in settings where contradiction should not entail triviality.

We then extended the framework in three directions. The type-2 extension (MFL-T2) captures second-order uncertainty about truth, falsity, hesitation, and contradiction via interval type-2 truth values, supporting both type-reduced (crisp) and envelope (interval) interpretations. The type-3 extension (MFL-T3) organizes mediative truth in a granular, multi-level form, enabling explicit cross-granule aggregation policies that reflect domain requirements such as robustness to outliers or priority of trusted sources. Finally, the quantum extension (QMFL) embeds mediative reasoning in the effect-algebraic setting of quantum theory and connects the classical mediative score to a Born expectation of an adaptive mediative effect, thereby opening a path to quantum implementations and to effect-based quantum granular computing architectures.

The autonomous-driving sensor-fusion example demonstrates how mediative truth degrees can be interpreted conservatively in a safety-first regime and linked to concrete decision policies. It further illustrates that the higher-type and quantum variants reduce to type-1 behaviour under appropriate modelling assumptions, while retaining a clear route to richer behaviours when uncertainty, granular heterogeneity, or non-commuting effects become operationally significant.

Future research directions include completeness and representation theorems for the extended logics, development of proof calculi tailored to mediative reasoning, and further applications in control, diagnosis, decision support, and quantum technologies.
\section{Notation and Symbol Table} 
For convenience, Table~\ref{tab:mfl-symbols-extended} collects the main syntactic, semantic, and operational symbols used throughout the paper. We also recall the distinction between syntactic derivability, semantic entailment, and (mediative) validity in the present setting.

The sequent $\Gamma \vdash_m \varphi$ denotes syntactic derivability in the propositional system MFL-T1: the formula $\varphi$ can be obtained from the set of premises $\Gamma$ by a finite sequence of applications of the axioms and inference rules of the logic. This is a proof-theoretic notion and does not involve semantic evaluation.

Semantic entailment is written $\Gamma \models_m \varphi$. It means that for every mediative valuation $v$, if all premises $\psi\in\Gamma$ attain maximal mediative degree (i.e., $M(v(\psi))=1$), then so does $\varphi$ (i.e., $M(v(\varphi))=1$). This notion depends on the mediative semantics through the map $\varphi \mapsto v(\varphi)\in V$ and the induced scalar evaluation $M$.

Finally, $\models_m \varphi$ denotes mediative validity: $\varphi$ is valid if $M(v(\varphi))=1$ for every mediative valuation $v$, equivalently $\emptyset \models_m \varphi$. With these notions in place, soundness states that every formula derivable in MFL-T1 is semantically entailed under the mediative semantics.

\begin{table}[h]
\centering\scriptsize
\caption{Logical, semantic, and operational symbols used across MFL-T1, MFL-T2, MFL-T3, and QMFL.}
\label{tab:mfl-symbols-extended}
\renewcommand{\arraystretch}{1.15}
\begin{tabular}{lll}
\hline
\textbf{Symbol} & \textbf{Meaning} & \textbf{Illustrative example} \\
\hline

$\varphi,\psi$ 
& Propositional formulas 
& $\varphi := \text{``obstacle detected''}$ \\

$p,q$ 
& Atomic propositions 
& $p := \text{``radar detects obstacle''}$ \\

$\top$ 
& Truth constant (tautology) 
& $\top := (p \to p)$ \\

$\bot$ 
& Falsity constant 
& $\bot := \neg \top$ \\

$\neg$ 
& Negation 
& $\neg \varphi$ \\

$\wedge$ 
& Conjunction 
& $\varphi \wedge \psi$ \\

$\vee$ 
& Disjunction 
& $\varphi \vee \psi$ \\

$\to$ 
& Implication (residuum-based) 
& $\varphi \to \psi$ \\

$\leftrightarrow$ 
& Bi-implication 
& $(\varphi \to \psi)\wedge(\psi \to \varphi)$ \\

$\Med(\varphi)$ 
& Mediative connective applied to $\varphi$ 
& ``mediated evaluation of $\varphi$'' \\

$v$ 
& Mediative valuation 
& $v(\varphi) = (\mu_\varphi,\nu_\varphi)$ \\

$(\mu_\varphi,\nu_\varphi)$ 
& Truth and falsity degrees (type-1) 
& $(0.5,0.0)$ \\

$\pi,\zeta$ 
& Hesitation and contradiction degrees 
& $\pi=\max(0,1-\mu-\nu),\;\zeta=\max(0,\mu+\nu-1)$ \\

$M(\mu,\nu)$ 
& Scalar mediative evaluation 
& $M(0.5,0.0)=0.75$ \\

$\Gamma$ 
& Set of formulas (theory) 
& $\Gamma=\{\varphi,\varphi\to\psi\}$ \\

$\vdash_m$ 
& Formal derivability in MFL-T1 (proof-theoretic) 
& $\Gamma \vdash_m \varphi$ \\

$\models_m$ 
& Semantic entailment under mediative semantics 
& $M(v(\psi))=1\ \forall\psi\in\Gamma \Rightarrow M(v(\varphi))=1$ \\

$\models_m \varphi$ 
& Mediative validity 
& $M(v(\varphi))=1$ for all mediative valuations $v$ \\

\hline
\multicolumn{3}{c}{\textbf{Type-2 Mediative Fuzzy Logic (MFL-T2)}} \\
\hline

$\tilde{\mu},\tilde{\nu}$ 
& Interval type-2 truth and falsity fuzzy sets 
& e.g., $\mathrm{Proj}^{U}(\tilde{\mu})=[0.6,0.8]$ \\

$\underline{\mu},\overline{\mu}$ 
& Lower and upper bounds of truth 
& $\mu\in[\underline{\mu},\overline{\mu}]$ \\

$\underline{\nu},\overline{\nu}$ 
& Lower and upper bounds of falsity 
& $\nu\in[\underline{\nu},\overline{\nu}]$ \\

$T$ 
& Chosen $t$-norm 
& e.g., $T(a,b)=\max(0,a+b-1)$ (Łukasiewicz) \\

$S$ 
& $t$-conorm dual to $T$ 
& e.g., $S(a,b)=\min(1,a+b)$ (Łukasiewicz) \\

$\Rightarrow_T$ 
& Residuum induced by $T$ 
& $a\Rightarrow_T b := \sup\{c\in[0,1]\mid T(a,c)\le b\}$ \\

$H(\mu,\nu)$ 
& Hesitation function 
& $H=\max(0,1-\mu-\nu)$ \\

$C(\mu,\nu)$ 
& Contradiction function 
& $C=\max(0,\mu+\nu-1)$ \\

$H_L,H_U$ 
& Lower/upper hesitation bounds 
& $H_L=\max(0,1-\overline{\mu}-\overline{\nu}),\ H_U=\max(0,1-\underline{\mu}-\underline{\nu})$ \\

$C_L,C_U$ 
& Lower/upper contradiction bounds 
& $C_L=\max(0,\underline{\mu}+\underline{\nu}-1),\ C_U=\max(0,\overline{\mu}+\overline{\nu}-1)$ \\

$\bar{M}_p$,\ $[M_L(p),M_U(p)]$ 
& Type-reduced / envelope mediative degrees 
& $\bar{M}_p := M(\bar{\mu}_p,\bar{\nu}_p)$,\quad $[M_L,M_U]$ \\

\hline
\multicolumn{3}{c}{\textbf{Type-3 Granular Mediative Fuzzy Logic (MFL-T3)}} \\
\hline

$g$ 
& Granule index (source/context) 
& $g=\text{radar},\text{camera}$ \\

$G$ 
& Set of granules 
& $G=\{g_{\mathrm{radar}},g_{\mathrm{cam}}\}$ \\

$v_g(\varphi)$ 
& Local mediative valuation at granule $g$ 
& $v_g(\varphi)=(\mu_{\varphi,g},\nu_{\varphi,g})$ \\

$M_g(\varphi)$ 
& Local scalar mediative degree 
& $M_g(\varphi):=M(\mu_{\varphi,g},\nu_{\varphi,g})$ \\

$A_\varphi$ 
& Granular aggregation operator 
& $M_G(\varphi):=A_\varphi((M_g(\varphi))_{g\in G})$ \\

$M_G(\varphi)$ 
& Global mediative degree after aggregation 
& $M_G(\varphi)\in[0,1]$ \\

\hline
\multicolumn{3}{c}{\textbf{Quantum Mediative Fuzzy Logic (QMFL)}} \\
\hline

$\rho$ 
& Quantum state (density operator) 
& $\rho\succeq 0,\ \mathrm{Tr}(\rho)=1$ \\

$E$ 
& Quantum effect 
& $0 \preceq E \preceq I$ \\

$E_p^{+},E_p^{-}$ 
& Positive/negative evidence effects for $p$ 
& $E_p^{+}\neq I-E_p^{-}$ allowed \\

$\mu_p(\rho),\nu_p(\rho)$ 
& Born degrees (truth/falsity channels) 
& $\mu_p(\rho)=\mathrm{Tr}(\rho E_p^{+}),\ \nu_p(\rho)=\mathrm{Tr}(\rho E_p^{-})$ \\

$w_{1,p}(\rho),w_{2,p}(\rho)$ 
& Mediative weights 
& $w_{1,p}=1-\pi_p-\zeta_p/2,\ w_{2,p}=\pi_p+\zeta_p/2$ \\

$\pi_p(\rho)$ 
& Quantum hesitation degree 
& $\pi_p=\max(0,1-\mu_p-\nu_p)$ \\

$\zeta_p(\rho)$ 
& Quantum contradiction degree 
& $\zeta_p=\max(0,\mu_p+\nu_p-1)$ \\

$M_p(\rho)$ 
& Quantum mediative effect for $p$ 
& $M_p(\rho)=w_{1,p}E_p^{+}+w_{2,p}(I-E_p^{-})$ \\

$M_q(p,\rho)$ 
& Quantum mediative degree (Born expectation) 
& $M_q(p,\rho)=\mathrm{Tr}(\rho\,M_p(\rho))$ \\

\hline
\end{tabular}
\end{table}



\section*{Declaration of competing interest}
The author declares no known competing financial interests or personal relationships that could have appeared to influence the work reported in this paper.



\section*{Data availability}
No datasets were generated or analyzed during the current study.
All theoretical results and examples are fully described in the article.

\section*{Declaration of generative AI and AI-assisted technologies in the writing process}
The author declares that no generative AI or AI-assisted technologies were used to write the manuscript or to generate figures or analytical
results. Only standard text editing tools were employed.

\bibliographystyle{elsarticle-num}
\bibliography{mfl_refs}

\end{document}